\documentclass[letterpaper]{article} %
\usepackage{aaai2026}  %
\usepackage{times}  %
\usepackage{helvet}  %
\usepackage{courier}  %
\usepackage[hyphens]{url}  %
\usepackage{graphicx} %
\usepackage{natbib}  %
\usepackage{caption} %
\usepackage{algorithm}
\usepackage{algorithmic}

\usepackage{newfloat}
\usepackage{listings}
\DeclareCaptionStyle{ruled}{labelfont=normalfont,labelsep=colon,strut=off} %
\floatstyle{ruled}
\newfloat{listing}{tb}{lst}{}
\floatname{listing}{Listing}
\usepackage{amsmath}
\usepackage{amsthm}
\usepackage{amssymb}
\usepackage{subcaption}
\usepackage{booktabs}
\usepackage{xcolor}
\usepackage{multirow} 
\usepackage{makecell}
\usepackage{pifont}
\usepackage{enumitem}
\usepackage{xspace}

\newtheorem{theorem}{Theorem}[section]

\newtheorem{proposition}{Proposition}[section]
\newtheorem{lemma}{Lemma}[section]

\newtheorem{observation}{Observation}[section]
\numberwithin{algorithm}{section}

\newcommand{\method}{LaT-IB\xspace}
\newcommand{\methods}{LaT-IB's\xspace}
\newcommand{\periods}{Warmup, Knowledge Injection and Robust Training\xspace}

\title{Is the Information Bottleneck Robust Enough? \\
Towards Label-Noise Resistant Information Bottleneck Learning}
\author {
    Yi Huang\textsuperscript{\rm 1},
    Qingyun Sun\textsuperscript{\rm 1}\thanks{Corresponding Author},
    Yisen Gao\textsuperscript{\rm 2},
    Haonan Yuan\textsuperscript{\rm 1},
    Xingcheng Fu\textsuperscript{\rm 3},
    Jianxin Li\textsuperscript{\rm 1}
}
\affiliations {
    \textsuperscript{\rm 1}SKLCCSE, School of Computer Science and Engineering, Beihang University, Beijing, China\\
    \textsuperscript{\rm 2}Department of Computer Science and Engineering, HKUST, Hong Kong, China\\
    \textsuperscript{\rm 3}Key Lab of Education Blockchain and Intelligent Technology, Ministry of Education, Guangxi Normal University, China\\
    \{yihuang, sunqy, yuanhn, lijx\}@buaa.edu.cn, 
    ygaodi@cse.ust.hk, 
    fuxc@gxnu.edu.cn
}

\usepackage{bibentry}

\begin{document}

\maketitle

\begin{abstract}
The Information Bottleneck (IB) principle facilitates effective representation learning by preserving label-relevant information while compressing irrelevant information. 
However, its strong reliance on accurate labels makes it inherently vulnerable to label noise, prevalent in real-world scenarios, resulting in significant performance degradation and overfitting.
To address this issue, we propose \textbf{\method}, a novel \textit{\textbf{La}bel-Noise Resistan\textbf{T} \textbf{I}nformation \textbf{B}ottleneck} method which introduces a \textit{``Minimal-Sufficient-Clean"} (MSC) criterion. Instantiated as a mutual information regularizer to retain task-relevant information while discarding noise, MSC addresses standard IB’s vulnerability to noisy label supervision.
To achieve this, \method employs a noise-aware latent disentanglement that decomposes the latent representation into components aligned with to the clean label space and the noise space. 
Theoretically, we first derive mutual information bounds for each component of our objective including prediction, compression, and disentanglement, and moreover prove that optimizing it encourages representations invariant to input noise and separates clean and noisy label information.
Furthermore, we design a three-phase training framework: \periods, to progressively guide the model toward noise-resistant representations. 
Extensive experiments demonstrate that \method achieves superior robustness and efficiency under label noise, significantly enhancing robustness and applicability in real-world scenarios with label noise.
\end{abstract}

\section{Introduction}
\label{sec:intro}
The Information Bottleneck (IB) principle~\cite{IB} provides a fundamental theoretical framework for balancing compression and relevance in representation learning.
Rooted in information theory, it has increasingly influenced the development of deep learning~\cite{IB-survey}. 
IB encourages representations $Z$ that retain only task-relevant information while discarding irrelevant or redundant input features using Mutual Information (MI) $I(\cdot; \cdot)$:
\begin{equation}
    \min -I(Y;Z)+\beta I(X;Z).
\end{equation}
IB-based methods aim to extract ``Minimal-Sufficient" representations, inherently filtering out input noise and spurious correlations. 
This selective encoding mechanism contributes to their notable robustness under noisy or adversarial input perturbations~\cite{x_robust1}.

However, input noise rarely eliminates all useful information, allowing IB to extract meaningful features from $Y$. In contrast, label noise corrupts the supervisory signal, causing $I(Y;Z)$ to mislead $Z$ to fit incorrect labels, thereby reducing robustness.
This vulnerability is critical in real-world settings, where label noise is common and can severely harm performance, as real graphs are often disturbed by noise and unexpected factors.~\cite{li2025simplified}. To address this, Label-Noise Representation Learning (LNRL)~\cite{LN-survey} aims to extract robust features despite label corruption. 

\begin{table}[t]
  \centering
    {\small
    \begin{tabular}{c|c|c}
    \toprule
    CIFAR10 & 40\% asym noise & 50\% sym noise \\
    \midrule
    ResNet34 & 77.78\% & 79.4\% \\
    VIB ($\beta=0.01$) & 73.80\% & 10.0\% \\
    \midrule
    Cora (40\% noise) & Epoch: $0 \to 20$ & Epoch: $20 \to 100$ \\
    \midrule
    \multirow{2}{*}{GIB}    & 22.9\% $\to$ 69.5\% & 69.5\% $\to$ 55.1\% \\
    & steady increase $\uparrow$ & steady decline $\downarrow$ \\
    \bottomrule
    \end{tabular}
    }
  \caption{Performance of IB methods under noise conditions.}
  \label{tab:example}
\end{table}

To empirically test the hypothesis that \textbf{IB is inherently vulnerable to label noise}, we conduct preliminary experiments on two tasks: image classification in computer vision and node classification in graph learning. 
We evaluate two representative IB-based methods: VIB~\cite{VIB} and GIB~\cite{GIB}.
As shown in Table~\ref{tab:example}, VIB suffers performance drops and even training collapse, while GIB exhibits degraded accuracy. See Appendix E.3 for details.

To mitigate this, a simple remedy is to denoise the labels prior to applying IB. However, this two-stage pipeline is inherently suboptimal in both theory and practice.
\begin{theorem}[Cumulative Degradation]
\label{theo:Degradation}
    In the two-stage approach, $f_1$ is used to modify the labels $Y'=f_1(\mathcal{D})$, and $f_2$ is responsible for extracting valid information from $\mathcal{D}$ while approximating the prediction result to $f_1(\mathcal{D})$ . For one-stage model $g(\mathcal{D})$, it extracts the relevant information while removing noise. If the denoising abilities of $f_1$ and $g$ are the same, the following inequality holds:
    \begin{equation}
        P(f_2(\mathcal{D})\neq g(\mathcal{D})) \geq \frac{H(Y'|\mathcal{D})-1}{\log(|\mathcal{Y}| - 1)},
    \end{equation}
    where $\mathcal{Y}$ denotes the support of $Y$, and $|\mathcal{Y}|$ denotes the number of elements in  $\mathcal{Y}$. The two models perform identically iff $f_2$ achieves the error lower bound and $H(Y'|\mathcal{D}) = 0$.
\end{theorem}

The proof of Theorem~\ref{theo:Degradation} is given in Appendix C.1.
It demonstrates that cascading a denoising model $f_1$ with an IB learner $f_2$ leads to cumulative information loss compared with a unified model $g$, due to the extended information path. 
This phenomenon is further validated by empirical results, which show a clear degradation in the denoising effect when models are cascaded. 
See Appendix E.3 for detailed results.

\textbf{Core Issue:} \textit{How can the IB principle be effectively applied to real-world scenarios with complex and unknown label noise, in order to learn representations that are \textbf{both ``Minimal-Sufficient" and robust} to noisy supervision?}

\begin{figure}[t]
\centering
\includegraphics[width=0.42\textwidth]{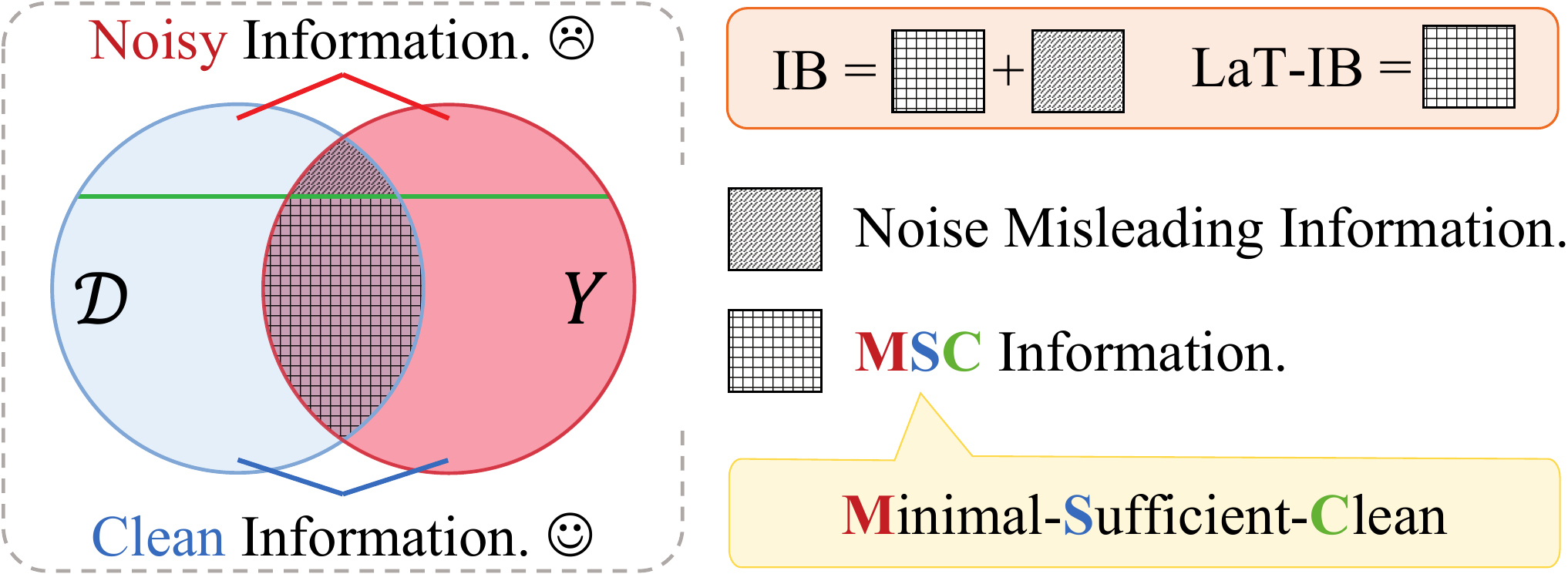}
\caption{Comparison between \method and IB Principle.}
\label{fig:target}
\end{figure}

Due to unknown label noise and the difficulty of integrating denoising with Information Bottleneck, applying IB in practice requires confronting the following key challenges:
\begin{itemize}
    \item How to formulate the IB objective under label noise to learn clean representation. ($\triangleright$ Section \ref{sec:4.1})
    \item How to optimize MI under label noise that distorts task-relevant representation learning. ($\triangleright$ Section \ref{sec:bound})
    \item How to effectively disentangle clean and noisy representations without knowing noisy samples. ($\triangleright$ Section \ref{sec:4.3})
\end{itemize}

\textbf{Present work.} To address the core issue and tackle the key challenges, we propose a \textit{\textbf{La}bel-Noise Resistan\textbf{T} \textbf{I}nformation \textbf{B}ottleneck} (\textbf{\method}) method. Centered on the idea of disentangling representations into clean and noisy label spaces, we formulate an IB training objective tailored for noisy supervision and theoretically justify its effectiveness through upper and lower bound analysis. To this end, we design a three-phase training framework: \periods, which gradually guides the model to learn ``\textit{Minimal-Sufficient-Clean}" (MSC) representations. A comparison between \method and standard IB principle is illustrated in Figure \ref{fig:target}.
Our contributions are:
\begin{itemize}
    \item We identify the inherent vulnerability of IB to label noise and prove that denoising before IB is suboptimal.
    \item We propose a \method method that introduces MSC criterion of representations to enhance IB’s robustness to label noise while maintaining its essential characteristics.
    \item We provide theoretical upper and lower bounds for \method, showing how disentangling clean and noise features enables robust representation learning. Based on this, we design a principled model and training framework.
    \item Extensive experiments evaluate \methods robustness and efficiency, outperforming baselines under label noise and adversarial attacks across diverse tasks and domains.
\end{itemize}

\section{Related Work}
\label{sec:related_work}
\subsection{Information Bottleneck for Robustness}
The IB~\cite{IB} framework introduces a feature learning paradigm grounded in information theory. Works such as VIB~\cite{VIB} and GIB~\cite{GIB} have advanced its practical use. Considering robustness, methods like DisenIB~\cite{DisenIB} and DGIB~\cite{DGIB} show reasonable robustness to input features, with studies~\cite{DT-JSCC, pensia2020extracting} further improving resilience to input noise.

Considering the presence of label noise, RGIB~\cite{link-GIB}  explores structural noise in GNNs to improve link prediction robustness.
However, comprehensive studies on the vulnerability of IB to label noise still remain lacking.

\subsection{Label-Noise Representation Learning}
The LNRL aims to improve model robustness and representation quality under noisy label conditions. 
Existing approaches for learning with noisy labels include sample selection~\cite{sample-selection, JoCoR}, which filters out likely noisy samples; robust loss functions~\cite{GCE, SCE}, which modify loss terms to reduce sensitivity to incorrect labels; noise-robust architectures~\cite{ELR}, which use regularization to avoid overfitting noise; and data augmentation, such as mixup-based methods~\cite{mixup, fmix}, which interpolates samples to improve generalization.

However, most methods ignore representation-level constraints, making it hard to learn task-relevant and noise-invariant features under severe noise or distribution shifts.

\begin{figure*}[!t]
\begin{subfigure}[b]{0.215\linewidth}
    \centering
    \includegraphics[width=\linewidth]{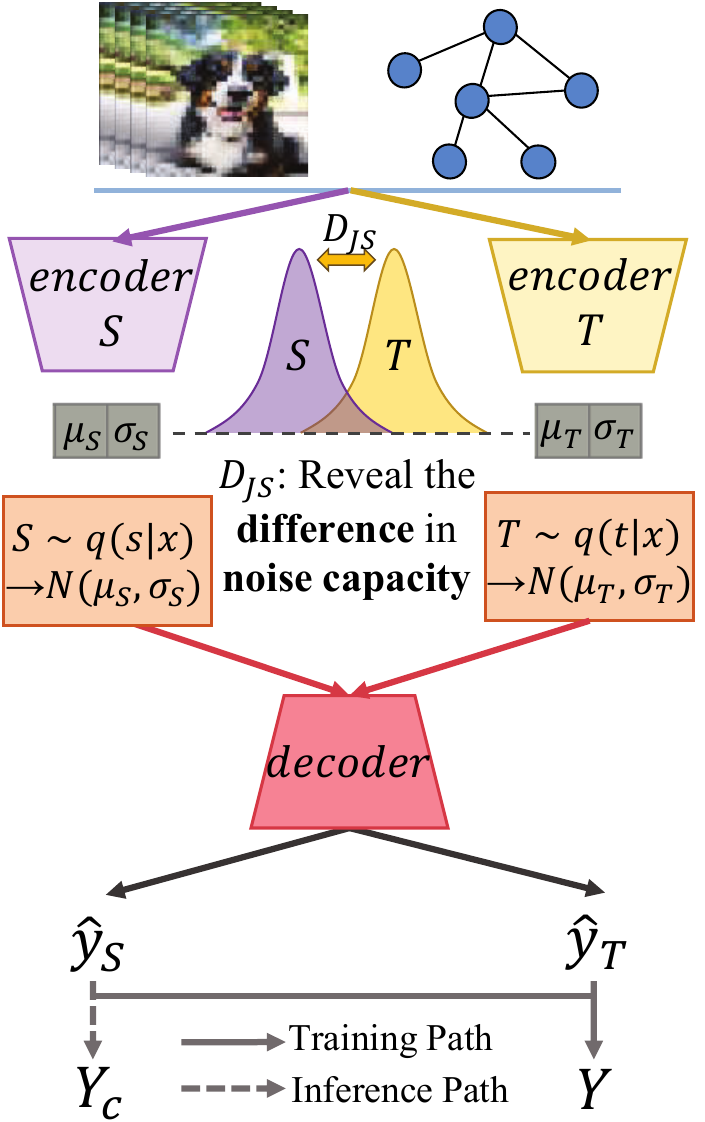}
\end{subfigure}
\begin{subfigure}[b]{0.78\linewidth}
    \centering
    \includegraphics[width=\linewidth]{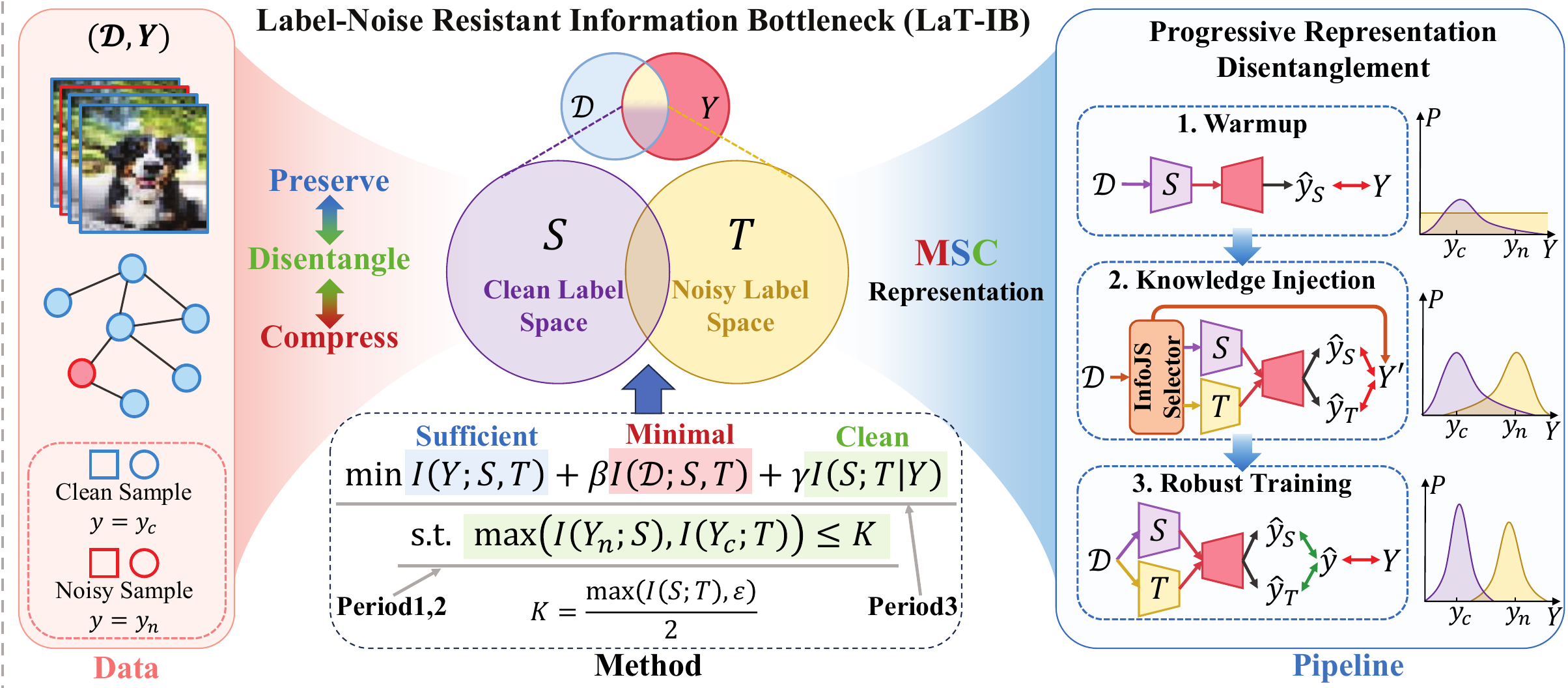}
\end{subfigure}
\caption{
Left: The overall \method model architecture with dual encoders for extracting features from clean ($S$) and noisy ($T$) label spaces, and a shared decoder.
Right: An illustration of the \method method,  which disentangles representations to extract ``Minimal-Sufficient-Clean" features. Specifically, its pipeline consists of three period: \periods, which transform Eq.~\eqref{core} from a theoretical formulation into a practical training procedure.
}
\label{fig:framework}
\end{figure*}

\section{Preliminary Analysis}
\label{sec:Preliminary}
\subsubsection{Notation.}
We primarily define the input data $\mathcal{D}$. 
For vision tasks, $\mathcal{D} = X$, where $X \in \mathbb{R}^{N \times C \times P \times Q}$ denotes $N$ samples with $C$ channels and spatial size $P \times Q$ (e.g., height $\times$ width). For graph learning tasks, $\mathcal{D} = \mathcal{G} = (X, A)$, where $X \in \mathbb{R}^{N \times d}$ denotes $d$-dimensional features for $N$ nodes and $A \in \mathbb{R}^{N \times N}$ represents the adjacency matrix. Each sample $\xi_i \in \mathcal{D}$ has a  label $y_i \in Y$, which may be corrupted by noise  during the labeling process. 
We denote $Y_c$ and $Y_n$ as the clean and noisy counterparts of $Y$ respectively.

\subsubsection{Analysis of IB Theory with Label Noise.}
In the traditional IB, $I(X;Z)$ encourages minimal representations by compressing the input, while $I(Y;Z)$ ensures sufficiency by preserving task-relevant information. 
However, when the label $Y$ is corrupted by noise, maximizing $I(Y;Z)$ is equivalent to maximizing $I(Y_c,Y_n;Z)$, which inadvertently causes the learned representation $Z$ to capture noise $Y_n$, thus compromising robustness and degrading performance.

In this study, we aim to mitigate the negative impact of label noise on model performance while preserving the ``Minimal-Sufficient" property of the IB method. 
Ideally, we consider a robust IB method $\mathcal M_{IB}$ that, given a dataset $(\mathcal{D}, Y)$ where $Y$ consists of both clean labels $y_i \in Y_{c}$ and noisy labels $y_j \in Y_{n}$, aims to satisfy the following objective:
\begin{equation}
\label{prob_def}
    \begin{gathered}
       \min -I(Z;Y_{c}) + \beta I(Z;\mathcal{D})\\
        \text{s.t.} \  Z = \mathcal{M}_{IB}(\mathcal{D}, Y).
    \end{gathered}
\end{equation}
Compared to the traditional IB objective, the goal of Eq. \eqref{prob_def} is to maximize the MI between the learned representation and the clean labels $Y_c$ , rather than with all observed labels $Y$. 
However, \textbf{whether each label is clean or noisy is unknown}. In the next section, we introduce a concrete solution to mitigate IB’s vulnerability to label noise.

\section{Methodology}
In this paper, we propose Label-Noise Resistant Information Bottleneck (\method), along with theoretical formulation, model architecture and a tailored training framework, as illustrated in Figure~\ref{fig:framework}. 
We begin by presenting the formal objective of \method and interpreting its theoretical implications.  
To enable efficient optimization, we derive upper and lower bounds that simplify the objective, effectively bridging the gap between theory and practice.
Finally, drawing on key insights, we design a three-phase training framework: \periods, clarify the role of each phase and facilitate the progressive disentanglement of clean and noise-related representations.

\subsection{Label-Noise Resistant Information Bottleneck}
\label{sec:4.1}
In real-world datasets, each training sample may have either a clean or a corrupted label, and sometimes both possibilities coexist probabilistically. Using a unified representation for all samples under such ambiguity can cause conflicting features and hurt downstream tasks. To mitigate this, we disentangle the representation into two parts: $S$ under the clean label space, and $T$ under the noise space. Under this disentanglement, the objective in Eq. \eqref{prob_def} can be reformulated as:
\begin{equation}
    \min -I(S; Y_{c}) + I(\mathcal{D}; S, T).
\end{equation}

Since only $Y$ are available in the dataset, we implicitly associate it with the joint representation of $S$ and $T$, where disentanglement is encouraged by $\min I(S;T|Y)$. A successful disentanglement implies that $S$ and $T$ encode conditionally independent given $Y$, capturing distinct semantics.
With $\beta$ and $\gamma$ as balancing factors, the \method is formulated as:
\begin{equation}
\label{core-before}
    \min \underbrace{-I(Y;S,T)}_{\text{prediction term}}+\beta \underbrace{I(\mathcal{D};S,T)}_{\text{compression term}}+\gamma \underbrace{I(S;T| Y)}_{\text{disentanglement term}},
\end{equation}

However, Eq. \eqref{core-before} still cannot map $S$ to clean features and $T$ to noise features. To address this and further explore its representational meaning, we introduce two lemmas below.

\begin{lemma}[Nuisance Invariance]
    \label{lemma1}
    Taking the part of $\mathcal{D}$ that does not contribute to $Y$ as $\mathcal{D}_n$ ($D_n$ is independent of $Y$), and considering the Markov chain $(Y,\mathcal{D}_n) \to \mathcal{D} \to (S,T)$, the following inequality holds:
    \begin{equation}
        I(\mathcal{D}_n;S,T)\leq -I(Y;S,T)+I(\mathcal{D};S,T).
    \end{equation} 
\end{lemma}

\begin{lemma}[Feature Convergence]
    \label{lemma2}
    Assuming that $Y$ can potentially contain all information about $Y_c$ and $Y_n$, the following inequality holds when $\max(I(Y_n;S), I(Y_c;T))\leq \max(I(S;T), \varepsilon) / 2 = K, \varepsilon>0, \varepsilon\in \mathbb{R}$ is satisfied: 
    \begin{equation}
        -I(Y_c;S)-I(Y_n;T)-\varepsilon\leq -I(Y;S,T)+I(S;T| Y).
    \end{equation}
\end{lemma}

The detailed proofs of these lemmas can be found in Appendix C.2. 
Lemma~\ref{lemma1} demonstrates that optimizing $\min -I(Y;S,T) + I(\mathcal{D};S,T)$ in Eq. \eqref{core-before} ($\beta = 1$) essentially reduces the model's tendency to learn features irrelevant to $Y$ (denoted as $\mathcal{D}_n$). Lemma~\ref{lemma2} further indicates that, when the MI terms $I(Y_n, S)$ and $I(Y_c, T)$ are sufficiently small, optimizing $\min -I(Y;S,T) + I(S,T|Y)$ in Eq. \eqref{core-before} ($\gamma = 1$) effectively strengthens the mapping relationships $S \rightarrow Y_c$ and $T \rightarrow Y_n$. 
Based on these insights, we can first ensure the conditions in Lemma \ref{lemma2} then optimize the main objective in Eq. \eqref{core-before} as a form of \textbf{progressive representation disentanglement}. This enables the model to separate clean and noisy features while avoiding learning irrelevant noise $\mathcal{D}_n$. 

By combining Lemma \ref{lemma1} and Lemma \ref{lemma2}, we obtain a principled training objective that integrates sufficiency, compression, and clean-noise disentanglement: 
\begin{equation}
\label{core}
    \begin{gathered}
    \min \underbrace{-I(Y;S,T)}_{\text{Sufficient}} + \beta \underbrace{I(\mathcal{D};S,T)}_{\text{Minimal}} + \gamma \underbrace{I(S;T|Y)}_{\text{Clean}} \\
    \quad \text{s.t. } \underbrace{\max(I(Y_n;S), I(Y_c;T)) \leq K }_{\text{Clean}}. \\
    \end{gathered}
\end{equation}

\subsection{Bound Analysis and Implementation}
\label{sec:bound}
Building on the formulation introduced in the previous section, we now turn to the optimization of the proposed objective in Eq.~\eqref{core}.
Since directly optimizing the multivariate MI is intractable, we first simplify the original objective by analyzing upper and lower bounds of MI, and then present the implementation strategy for each term. 
All proposition proofs are provided in the Appendix C.3.

\begin{proposition}[The upper bound of $-I(Y;S,T)$]
\label{prop1}
    Given the label $Y$ and the variable $S,T$ that learns the characteristics of the clean label space and the noisy label space respectively, we have:
    \begin{equation}
    \label{prop1-bound}
        -I(Y;S,T)\leq -\max\left(I(Y;S), I(Y;T)\right).
    \end{equation}
\end{proposition}

Intuitively, Eq. \eqref{prop1-bound} encourages encoders to focus on learning its own knowledge, ensuring consistency in the learned representation. 
Further, since MI terms are intractable, each $I(Y,Z)$ with $Z\in \{S,T\}$ is lower-bounded by the cross-entropy loss using a variational approximation $q_\theta (y|z)$:
\begin{equation}
\label{prop1-eq}
    I(Y;Z) \geq \mathbb{E}_{p(y,z)}\left(\log(q_\theta(y|z))\right) := -\mathcal{L}_{CE} (Z,Y),
\end{equation}

\begin{proposition}[The upper bound of $I(\mathcal{D};S,T)$]
\label{prop2}
    Let $\mathcal{D}$, $S$, $T$ be random variables. Assume the probabilistic mapping $p(\mathcal{D}, S, T)$ follows the Markov chain $S \leftrightarrow \mathcal{D} \leftrightarrow T$. Then:
    \begin{equation}
    \label{prop2-bound}
        I(\mathcal{D}; S, T) \leq I(\mathcal{D}; S) + I(\mathcal{D}; T).
    \end{equation}
\end{proposition}

The implementation of each term $I(\mathcal{D};\cdot)$ remains consistent with that in VIB~\cite{VIB} and GIB~\cite{GIB}, achieved by minimizing the KL divergence between the variational posterior $q(\cdot|\mathcal{D})$ and the prior $p(\cdot)$.

\begin{proposition}[Reformulation of $I(S,T| Y)$]
\label{prop3}
    Given the label $Y$ and the variable $S,T$, minimizing $I(S;T | Y)$ is equivalent to minimize $I(S,Y;T,Y)$.
\end{proposition}

The Proposition~\ref{prop3} achieves the tractable transformation of conditional MI theoretically. However, minimizing the term 
$I(S,Y;T,Y) = D_{\mathrm{KL}}\big[q(S,T,Y) || q(S,Y)q(T,Y)\big]$ is intractable since both distributions involve mixtures with many components. Therefore, we use the density-ratio trick~\cite{sugiyama2012density} by introducing a discriminator $d$, that learns to distinguish between samples from the joint distribution $q(s,t,y)$ and those from the product of marginals $q(s,y)q(t,y)$. In particular, we sample negative pairs $((s,y), (t,y'))$ from $q(s,y)q(t,y)$, where $(s,y)$ and $(t,y')$ are drawn independently, and positive pairs $((s,y), (t,y))$ from the joint distribution $q(s,t,y)$, where both $s$ and $t$ correspond to the same sample. The discriminator $d((s,y), (t,y'))$ is trained to output the probability that a given pair comes from the joint distribution, and the objective is to minimize the MI by solving the following problem:
\begin{equation}
\label{dis-eq}
\begin{gathered}
     \min_q \max_d \mathbb{E}_{q(s,y)q(t,y)}\log d((s,y),(t,y')) \\ 
     +\mathbb{E}_{q(s,t,y)} \log(1-d((s,y),(t,y))).
\end{gathered}
\end{equation}
When the discriminator cannot distinguish between joint and independent samples, the MI is effectively minimized.

\begin{proposition}[Reformulation of the condition in Eq.~\eqref{core}: $\max(I(Y_n;S),I(Y_c;T))\leq K$]
\label{prop4}
    Minimizing $I(Y_c; T)$ and $I(Y_n; S)$ is equivalent to maximize $I(Y_n; T)$ and $I(Y_c; S)$.
\end{proposition}

Proposition \ref{prop4} relaxes the condition in Eq. \eqref{core}. Since the original MI calculation is mismatched and thus intractable, the relaxed formulation provides a tractable alternative that can be optimized efficiently, as described in Eq. \eqref{prop1-eq}.

\subsection{Principle to Practice: \method Framework}
\label{sec:4.3}
Based on the theoretical analysis above, this section introduces the practical implementation of \method.
To optimize the objective in Eq. \eqref{core}, we adopt a three-phase training framework to \textbf{progressively disentangle the representation}.
Specifically, we first introduce a \textbf{Warmup} period to provide the model with initial discriminative ability. Building on this, a \textbf{Knowledge Injection} period enforces the constraint by applying InfoJS selector, guiding the learning of $\text{encoder}_{S/T}$ via selected samples. Finally the \textbf{Robust Training} period focuses on optimizing the complete objective with prior knowledge, refining the model’s robustness.

\subsubsection{Feature-Decomposed Dual Encoder Architecture Design.}
Based on the Observation~\ref{obser1}, we adopt the Jensen-Shannon (JS) divergence as a metric to evaluate the noise retention capacity of the two encoders.

\begin{observation}
\label{obser1}
    With the decoder kept fixed, we train the encoder using datasets that share the same input $X$ but differ in the level of label noise in $Y$. As the noise gap between the two datasets increases, the divergence between the resulting encodings from the encoder also becomes larger.
\end{observation}

Accordingly, Figure~\ref{fig:framework} illustrates the overall architecture of the \method: the model is designed with a dual-encoder, single-decoder framework, where the two encoders extract features $S$ and $T$, respectively. Each encoder maps the input features to a high-dimensional Gaussian distribution, and the embeddings are sampled using the reparameterization trick.

\subsubsection{Phase 1: Warmup with Discriminative Learning under Noise.}
To address the problem of noise memorization during training, we introduce a Warmup phase where the model \textbf{builds basic discriminative ability}. Specifically, we pretrain the clean encoder $\text{encoder}_S$ using the full dataset, providing a foundation for more effective separation of clean and noisy samples in subsequent stages. The loss function in Warmup period is defined based on prediction $\hat y_S$:
\begin{equation}
\label{loss1}
    \mathcal L_{Warmup}=\mathcal{L}_{CE}(\text{decoder}(S),Y)=\mathcal{L}_{CE}(\hat y_S,y).
\end{equation}

\subsubsection{Noise-Aware Sample Selection.}
\label{sec:sel_method}
Since the variables $Y_c$ and $Y_n$ are unobservable, we approximate the constraint $\max(I(Y_n;S), I(Y_c;T)) \leq K$ in Eq.~\eqref{core} by selecting a partial set of confident samples to act as proxies for clean and noisy labels. Samples are then grouped into three categories for training: \textbf{Clean Set, Noise Set, and Uncertain Set}.

\begin{observation}
\label{obser2}
    For two different encoders, samples with more consistent predictions after passing through the decoder tend to have smaller divergence between their embeddings. In contrast, samples with inconsistent predictions correspond to larger embedding divergence.
\end{observation}

Observation \ref{obser2} suggests that the divergence between encoders can be used to identify clean samples. Moreover, prior studies~\cite{arpit2017a, song2019does} have shown that models tend to fit clean samples earlier. 
Based on these insights, we designed the \textbf{InfoJS selector} as detailed in Algorithm B.2, which identifies clean (noisy) samples as those with MI between $S$ and $Y$ being in the top $\delta$\% (bottom $\delta$\%) and JS divergence between $S$ and $T$ in the bottom $\delta$\% (top $\delta$\%), respectively. Unselected samples are treated as the uncertain set.
Labels are assigned as follows: $y' = y$ for Clean and Noise Sets, and $y' = g(\hat y_S) / g(\hat y_T)$ for Uncertain Set when training the $\text{encoder}_{S/T}$, where $g$ denotes either a debiasing function~\cite{menonlong} or one-hot mapping.

However, the InfoJS selector performs selection based on relative feature scores. To improve the quality of each set, we further enrich the sample composition by incorporating predicted confidence scores as an absolute criterion.

\subsubsection{Phase 2: Knowledge Injection to Disentangle Representations.}
\label{period2}
Once the model has acquired basic discriminative ability, we proceed to optimize the objective in Eq.~\eqref{core}. Given the condition and its reformulated form:
\begin{equation}
\begin{gathered}
    \underbrace{\max(I(Y_n;S), I(Y_c;T)) \leq K}_{\text{The original constraint in Eq. \eqref{core}}} \\
    \Rightarrow \underbrace{\max(I(Y_n;T)),\max( I(Y_c;S))}_{\text{The reformulated constraint in Proposition \ref{prop4}}},
\end{gathered}
\end{equation}
to satisfy the constraint, we introduce a Knowledge Injection phase to encourage the $\text{encoder}_{S,T}$ to learn disentangled representations. Furthermore, to enforce difference in noise representation between the two encoders, we incorporate the JS divergence $D_{JS}$ based on Observation~\ref{obser1}:
\begin{equation}
\label{loss20}
    \left\{
    \begin{aligned}
        &\mathcal L_{Clean}=\mathcal{L}_{CE}(\hat{y}_{S}, y')-D_{JS}(s\ \|\ t),\\
        &\mathcal L_{Uncertain}=\mathcal{L}_{CE}(\hat{y}_S, y') + \mathcal{L}_{CE}(\hat{y}_T, y') +D_{JS}(s\ \|\ t),\\
        &\mathcal L_{Noise}=\mathcal{L}_{CE}(\hat{y}_{T}, y')-D_{JS}(s\ \|\ t),
    \end{aligned}
    \right.
\end{equation}

For Clean and Noise Sets, divergence is maximized to increase encoder discrepancy; and for the Uncertain set, divergence is minimized to guide $T$ towards meaningful patterns. It is worth noting that the Uncertain set is much smaller, thus has limited influence on the encoders' training process.

Empirically, minimizing the $I(\mathcal{D};S,T)$ helps to leading to a more robust encoding space. To progressively disentangle the representation and achieve a minimal representation, we introduce a regularization term $\mathcal{L}_{Minimal}$ that approximates the $\min I(\mathcal{D}; S, T)$ term base on Proposition \ref{prop2}, and incorporate it into the loss function during the Knowledge Injection period to learn a compact representation:
\begin{equation}
\label{loss21}
    \mathcal L_{Injection}=\mathcal L_{Clean}+\mathcal L_{Uncertain}+\mathcal L_{Noise}+ \mathcal L_{Minimal}.
\end{equation}
This facilitates a smoother transition to the third Robust Training stage. The implementation details of $\mathcal{L}_{Minimal}$ are provided in the Appendix D.1.

\subsubsection{Phase 3: Robust Training for Representation Consistency.}
\label{period3}
The Warmup stage establishes initial discriminative ability, while Knowledge Injection realized constraint to guide the model toward informative and reliable samples. To further disentangle and enhance representation robustness under label noise, this period focuses on optimizing the full objective in Eq.~\eqref{core-before}: $\min -I(Y;S,T) + I(\mathcal{D}; S,T) + I(S;T|Y)$, aiming to learn noise consistent representations.

Section~\ref{sec:bound} has introduced the implementation of each objective term. Among them we propose $\mathcal{L}_{ConCE}$ to optimize the term $I(Y;S,T)$ based Eq. \eqref{prop1-bound} and \eqref{prop1-eq}:
\begin{equation}
\label{conCE}
    \mathcal{L}_{ConCE} \leftarrow \sum \min(\mathcal{L}_{CE}(\hat y_S, y), \mathcal{L}_{CE}(\hat y_T, y)),
\end{equation}
encouraging consistency between encoders and clean/noisy labels. Detailed formulations are provided in Appendix B.1.

The loss function for the Robust Training period is:
\begin{equation}
\label{loss3}
    \begin{aligned}
        \mathcal{L}_{Robust} = \frac{1}{|\mathcal B|} \sum_{i=1}^{\mathcal B} 
        [ 
        &\underbrace{\mathcal{L}_{ConCE}(\hat y_S, \hat y_T, y)}_{\text{Eq. \eqref{conCE}}} + \beta \underbrace{\mathcal{L}_{Minimal}}_{\text{Eq. \eqref{prop2-bound}}}\\
        &  - \gamma \underbrace{\log d(s_i, y_i; t_i, y_i)}_{\text{Proposition \ref{prop3}, Eq. \eqref{dis-eq}}}
        ],
    \end{aligned}
\end{equation}
where $\mathcal{B}$ denotes a training batch. In addition, we alternately update the discriminator $d$ based on Eq. \eqref{dis-eq}, using a random permutation  $\pi$ to approximate the marginal distribution:
\begin{equation}
\label{loss_d}
\begin{aligned}
    \mathcal L_d =\frac{1}{|\mathcal B|} \sum_{i=1}^{\mathcal B} &-\log (1-d(s_{i}, y_{i}; t_{\pi(i)}, y_{\pi(i)})) \\
    &-\log d(s_i, y_i; t_i, y_i).
\end{aligned}
\end{equation}

\section{Experiment}
\label{sec:exp}

\begin{table*}[t]
  \centering
    \tabcolsep=0.27cm
    {\small
  \begin{tabular}{cc|cccccc|c}
    \toprule
    \multirow{2}{*}{\textbf{Method}} & \multirow{2}{*}{\textbf{Model}} 
      & \multicolumn{5}{c}{\textbf{CIFAR-10N}} & \textbf{CIFAR-100N} & \multirow{2}{*}{\makecell{\textbf{Animal}\\\textbf{-10N}}} \\
    \cmidrule(r){3-7} \cmidrule(lr){8-8}
      & & aggre & rand1 & rand2 & rand3 & worst & noisy100 \\
    \midrule
    \multirow{2}{*}{\makecell{\textbf{Classic}\\\textbf{IB}}} & VIB & 86.11\textsubscript{$\pm$0.34} & 83.69\textsubscript{$\pm$0.50} & 83.69\textsubscript{$\pm$0.46} & 83.76\textsubscript{$\pm$0.29} & 73.80\textsubscript{$\pm$0.59} & 53.29\textsubscript{$\pm$0.09} & 76.28\textsubscript{$\pm$0.51}\\
    ~ & NIB & 85.21\textsubscript{$\pm$0.44} & 84.03\textsubscript{$\pm$1.43} & 81.98\textsubscript{$\pm$0.68} & 82.39\textsubscript{$\pm$0.43} & 73.51\textsubscript{$\pm$0.82} & 48.11\textsubscript{$\pm$0.40} & 75.62\textsubscript{$\pm$0.64}       \\
    \midrule
    \multirow{2}{*}{\makecell{\textbf{Robust}\\\textbf{Loss}}} & VIB ($\mathcal L_{GCE}$) & 85.70\textsubscript{$\pm$0.08} & 84.32\textsubscript{$\pm$0.50} & 83.97\textsubscript{$\pm$0.38} & 84.25\textsubscript{$\pm$0.68} & 78.88\textsubscript{$\pm$0.27}  & ---  & 81.72\textsubscript{$\pm$1.77}       \\
    ~ & VIB ($\mathcal L_{SCE}$) & 83.95\textsubscript{$\pm$0.10} & 82.65\textsubscript{$\pm$0.25} & 82.84\textsubscript{$\pm$0.31} & 82.50\textsubscript{$\pm$0.24} & 73.81\textsubscript{$\pm$1.54}  & 50.71\textsubscript{$\pm$0.14} & 77.17\textsubscript{$\pm$0.44} \\
    \midrule
    \multirow{2}{*}{\makecell{\textbf{Improved}\\\textbf{IB}}} & SIB & 89.99\textsubscript{$\pm$0.08} & 84.75\textsubscript{$\pm$1.04} & 85.07\textsubscript{$\pm$0.72} & 85.39\textsubscript{$\pm$0.50} & 70.58\textsubscript{$\pm$0.50} & 50.82\textsubscript{$\pm$0.41} & 83.95\textsubscript{$\pm$0.14}\\
    ~ & DT-JSCC & 85.46\textsubscript{$\pm$0.44} & 81.85\textsubscript{$\pm$0.66} & 81.14\textsubscript{$\pm$0.55} & 81.03\textsubscript{$\pm$0.34} & 69.73\textsubscript{$\pm$1.15} & 43.61\textsubscript{$\pm$0.19} & 78.98\textsubscript{$\pm$0.23} \\
    \midrule
    \multirow{3}{*}{\makecell{\textbf{Deniose}\\\textbf{+ IB}}} & JoCoR+VIB &86.39\textsubscript{$\pm$0.18} & 86.45\textsubscript{$\pm$0.02} & 86.53\textsubscript{$\pm$0.29} & 86.60\textsubscript{$\pm$0.11} & 81.65\textsubscript{$\pm$0.15} & 54.24\textsubscript{$\pm$0.18} & 75.45\textsubscript{$\pm$0.27}\\
    ~ & (ELR+)+VIB & \underline{92.65\textsubscript{$\pm$0.27}} & 92.09\textsubscript{$\pm$0.25} & 92.01\textsubscript{$\pm$0.20} & 91.93\textsubscript{$\pm$0.15} & 86.68\textsubscript{$\pm$0.25} & 61.06\textsubscript{$\pm$0.34} & \underline{85.87\textsubscript{$\pm$0.15}}\\
    ~ & Promix+VIB & 92.35\textsubscript{$\pm$0.38} & \underline{92.59\textsubscript{$\pm$0.40}} & \underline{92.42\textsubscript{$\pm$0.17}} & \underline{92.54\textsubscript{$\pm$0.21}} & \textbf{91.24\textsubscript{$\pm$0.28}} & \textbf{63.91\textsubscript{$\pm$0.19}} & 85.47\textsubscript{$\pm$0.51} \\
    \midrule
    \multirow{1}{*}{\textbf{Ours}} & \method & \textbf{94.17\textsubscript{$\pm$0.12}} & \textbf{93.25\textsubscript{$\pm$0.11}} & \textbf{93.19\textsubscript{$\pm$0.09}} & \textbf{93.03\textsubscript{$\pm$0.11}} & \underline{87.95\textsubscript{$\pm$0.22}} & \underline{63.59\textsubscript{$\pm$0.67}} & \textbf{88.49\textsubscript{$\pm$0.11}}\\
    \bottomrule
  \end{tabular}
  }
  \caption{Classification accuracy (\%) on the CIFAR-10N/100N and Animal-10N dataset. All the best results are highlighted in \textbf{bold}, and the second-best results are \underline{underlined}.}
  \label{tab:cifarN}
\end{table*}

\begin{table*}[t]
  \centering
    \setlength{\tabcolsep}{1mm}
    {\small
  \begin{tabular}{cc|ccccc|cccc}
    \toprule
    \multirow{2}{*}{\textbf{Method}} & \multirow{2}{*}{\textbf{Model}} & \multirow{2}{*}{\textbf{Clean}}
      & \multicolumn{4}{c|}{\textbf{Uniform Noise}} 
      & \multicolumn{4}{c}{\textbf{Pair Noise}} \\
    \cmidrule(lr){4-7} \cmidrule(l){8-11}
      & & & 10\% & 20\% & 30\% & 40\% & 10\% & 20\% & 30\% & 40\% \\
    \midrule
    \multirow{1}{*}{\makecell{\textbf{Classic}}} 
      & GIB & 71.57\textsubscript{$\pm$1.18} & \underline{70.50\textsubscript{$\pm$1.85}} & 64.30\textsubscript{$\pm$6.45} & 63.90\textsubscript{$\pm$3.51} & \underline{62.67\textsubscript{$\pm$1.35}} & 68.67\textsubscript{$\pm$3.47} & 61.30\textsubscript{$\pm$14.57} & 67.53\textsubscript{$\pm$4.77} & 55.57\textsubscript{$\pm$14.33} \\
    \midrule
    \multirow{2}{*}{\makecell{\textbf{Robust}\\\textbf{Loss}}}
      & GIB ($\mathcal L_{GCE}$) & 69.93\textsubscript{$\pm$0.69} & 67.43\textsubscript{$\pm$3.21} & 61.67\textsubscript{$\pm$7.19} & 47.80\textsubscript{$\pm$18.62} & 43.47\textsubscript{$\pm$14.50} & 50.93\textsubscript{$\pm$0.52} & 55.33\textsubscript{$\pm$11.23} & 62.37\textsubscript{$\pm$6.99} & 36.90\textsubscript{$\pm$15.23} \\
      & GIB ($\mathcal L_{SCE}$) & \underline{72.53\textsubscript{$\pm$0.12}} & 70.17\textsubscript{$\pm$2.10} & \underline{71.63\textsubscript{$\pm$2.05}} & 62.90\textsubscript{$\pm$8.09} & 51.87\textsubscript{$\pm$6.03} & 69.30\textsubscript{$\pm$1.66} & 68.23\textsubscript{$\pm$3.41} & 65.13\textsubscript{$\pm$5.02} & 51.13\textsubscript{$\pm$11.30} \\
    \midrule
    \multirow{1}{*}{\makecell{\textbf{Improv-}\\\textbf{ed IB}}} 
      & CurvGIB & 64.63\textsubscript{$\pm$5.28} & 65.67\textsubscript{$\pm$5.85} & 54.67\textsubscript{$\pm$10.09} & 54.00\textsubscript{$\pm$2.41} & 54.97\textsubscript{$\pm$2.78} & 59.97\textsubscript{$\pm$9.00} & 62.07\textsubscript{$\pm$5.15} & 66.63\textsubscript{$\pm$1.94} & 54.57\textsubscript{$\pm$1.25} \\
      & IS-GIB & 71.00\textsubscript{$\pm$1.22} & 69.97\textsubscript{$\pm$1.41} & 64.30\textsubscript{$\pm$2.30} & 59.77\textsubscript{$\pm$3.70} & 53.77\textsubscript{$\pm$4.41} & 64.83\textsubscript{$\pm$2.34} & 62.50\textsubscript{$\pm$1.51} & 62.50\textsubscript{$\pm$1.31} & 55.40\textsubscript{$\pm$4.74} \\
    \midrule
    \multirow{2}{*}{\makecell{\textbf{Denoise}\\\textbf{+ IB}}} 
      & RNCGLN+GIB & 70.57\textsubscript{$\pm$0.99} & 69.50\textsubscript{$\pm$0.86} & 63.43\textsubscript{$\pm$5.90} & 62.83\textsubscript{$\pm$4.15} & 53.27\textsubscript{$\pm$14.47} & \underline{69.90\textsubscript{$\pm$1.69}} & 68.20\textsubscript{$\pm$2.14} & 66.47\textsubscript{$\pm$3.21} & \underline{56.77\textsubscript{$\pm$15.11}} \\
      & CGNN+GIB & 71.87\textsubscript{$\pm$1.99} & 68.97\textsubscript{$\pm$3.09} & 65.47\textsubscript{$\pm$4.77} & \underline{64.93\textsubscript{$\pm$2.46}} & 48.83\textsubscript{$\pm$6.48} & 59.03\textsubscript{$\pm$12.67} & \underline{69.77\textsubscript{$\pm$1.77}} & \underline{68.50\textsubscript{$\pm$2.83}} & 53.93\textsubscript{$\pm$13.85} \\
    \midrule
    \textbf{Ours} & \method & \textbf{74.97\textsubscript{$\pm$0.68}} & \textbf{74.90\textsubscript{$\pm$2.09}} & \textbf{73.40\textsubscript{$\pm$2.62}} & \textbf{70.50\textsubscript{$\pm$3.86}} & \textbf{72.20\textsubscript{$\pm$4.22}} & \textbf{75.63\textsubscript{$\pm$0.46}} & \textbf{73.03\textsubscript{$\pm$1.77}} & \textbf{70.07\textsubscript{$\pm$3.20}} & \textbf{68.77\textsubscript{$\pm$2.29}} \\
    \bottomrule
  \end{tabular}
  }
  \caption{Classification accuracy (\%) on the Pubmed dataset under different noise types and noise rates. All the best results are highlighted in \textbf{bold}, and the second-best results are \underline{underlined}.}
  \label{tab:pubmed}
\end{table*}

In this section, we conduct extensive experiments to evaluate the robustness and efficiency of the \method under diverse tasks and various types of noise, including real-world and synthetic label noise, as well as adversarial perturbation.
\footnote{Code available at: https://github.com/RingBDStack/LaT-IB}

\subsection{Experimental Settings}
\subsubsection{Datasets.} We evaluate the proposed \method method on multiple datasets. For image classification, we utilize the CIFARN~\cite{CIFARN}, Animal-10N~\cite{animal-10N} and CIFAR~\cite{cifar} datasets. For node classification tasks,  we evaluate on Cora, Citeseer, Pubmed~\cite{G_data}, and DBLP~\cite{dblp}. More descriptions about datasets are provided in Appendix E.1.

\subsubsection{Baselines.} We compare our \method with four categories, 16 baselines in two scenarios: \ding{172} Classic IB methods; \ding{173} IB with robust loss functions; \ding{174} Improved IB variants; \ding{175} Two-stage denoising + IB methods. They comprehensively evaluate our \methods performance from multiple perspectives.

\subsubsection{Label Noise Settings.}
To evaluate the robustness of \method and baselines against label noise, we conduct experiments in both image and graph classification tasks.
For image classification, we evaluate on both real-world noisy datasets and synthetic settings with symmetric and asymmetric label noise, simulated using custom transition matrices as described in \cite{promix}. For node classification, we follow the protocol in \cite{noisyGL} to inject uniform and pairwise label noise into graph labels.

\subsubsection{Adversarial Attack Settings.} 
As discussed in Appendix D.1, the implementation of $I(\mathcal{D}; S, T)$ in our \method framework aligns with prior work VIB and GIB, thereby theoretically inheriting their robustness properties. To empirically verify this claim, we evaluate \methods performance under adversarial perturbations in the image classification setting.
Specifically, we adopt the FGSM~\cite{fgsm} attack to perturb input images with $\varepsilon \in \{0.05, 0.1, 0.2\}$, controlling the perturbation strength. Combined with noisy labels during training, this setup evaluate the robustness of model under compound noise conditions.

\subsection{Robustness Against Label Noise}
\begin{table*}[t]
  \centering
    \setlength{\tabcolsep}{2.5mm}
    {\small
  \begin{tabular}{c|cccc|cccc}
    \toprule
    \multirow{2}{*}{\textbf{Model}} & \multicolumn{4}{c|}{\textbf{CIFAR-10N (aggre)}} & \multicolumn{4}{c}{\textbf{CIFAR-10N (worst)}} \\
    \cmidrule(lr){2-5} \cmidrule(l){6-9}
    & No attack & $\varepsilon=0.05$ & $\varepsilon=0.1$ & $\varepsilon=0.2$ & No attack & $\varepsilon=0.05$ & $\varepsilon=0.1$ & $\varepsilon=0.2$ \\
    \midrule
    VIB & 86.11\textsubscript{$\pm$0.34} & 52.33\textsubscript{$\pm$1.55} & 43.18\textsubscript{$\pm$2.10} & 36.63\textsubscript{$\pm$1.10} & 73.80\textsubscript{$\pm$0.59} & 43.29\textsubscript{$\pm$2.36} & \underline{36.56\textsubscript{$\pm$3.28}} & 32.17\textsubscript{$\pm$3.27} \\
    VIB ($\mathcal{L}_{GCE}$) & 85.70\textsubscript{$\pm$0.08} & 54.15\textsubscript{$\pm$1.85} & 44.84\textsubscript{$\pm$3.00} & 34.72\textsubscript{$\pm$2.43} & 78.88\textsubscript{$\pm$0.27} & 43.27\textsubscript{$\pm$1.56} & 31.24\textsubscript{$\pm$1.42} & 24.23\textsubscript{$\pm$1.45} \\
    SIB & 89.99\textsubscript{$\pm$0.08} & \underline{56.48\textsubscript{$\pm$2.50}} & \underline{46.62\textsubscript{$\pm$2.41}} & \underline{38.14\textsubscript{$\pm$2.37}} & 70.58\textsubscript{$\pm$0.50} & 43.39\textsubscript{$\pm$2.83} & 33.40\textsubscript{$\pm$2.87} & \underline{27.89\textsubscript{$\pm$2.93}} \\
    (ELR+)+VIB & \underline{92.65\textsubscript{$\pm$0.27}} & 39.88\textsubscript{$\pm$0.74} & 23.16\textsubscript{$\pm$0.40} & 14.60\textsubscript{$\pm$0.39} & 86.68\textsubscript{$\pm$0.25} & 42.72\textsubscript{$\pm$0.24} & 26.44\textsubscript{$\pm$0.64} & 14.70\textsubscript{$\pm$0.70} \\
    Promix+VIB & 92.35\textsubscript{$\pm$0.38} & 51.27\textsubscript{$\pm$1.53} & 36.43\textsubscript{$\pm$0.65} & 20.49\textsubscript{$\pm$1.79} & \textbf{91.24}\textsubscript{$\pm$0.28} & \underline{52.88\textsubscript{$\pm$1.46}} & 36.05\textsubscript{$\pm$0.68} & 23.79\textsubscript{$\pm$0.14} \\
    \textbf{\method} & \textbf{94.17}\textsubscript{$\pm$0.12} & \textbf{69.38}\textsubscript{$\pm$1.23} & \textbf{60.66}\textsubscript{$\pm$2.03} & \textbf{49.64}\textsubscript{$\pm$2.27} & \underline{87.95\textsubscript{$\pm$0.22}} & \textbf{64.36}\textsubscript{$\pm$1.67} & \textbf{54.18}\textsubscript{$\pm$2.53} & \textbf{43.91}\textsubscript{$\pm$3.05} \\
    \bottomrule
  \end{tabular}
  }
  \caption{Classification accuracy (\%) on CIFAR-10N (aggre and worst) under different adversarial perturbation levels. 
For Denoise + IB methods, adversarial attacks are applied in both stages: the VIB model is trained using the output of a denoising model that has itself been attacked. 
All the best results are highlighted in \textbf{bold}, and the second-best results are \underline{underlined}.}
  \label{tab:cifar10n_attack}
\end{table*}

In this section, we evaluate the representation capability of our proposed method under various label noise conditions. Specifically, we investigate whether the \method model can effectively learn robust representations when trained on data corrupted by different types and levels of label noise. 

\textbf{Results.} In most scenarios, our proposed \method method outperforms other baseline approaches as shown in Table \ref{tab:cifarN} and \ref{tab:pubmed}. In certain cases, however, methods that first perform denoising and then apply IB achieve better results, likely due to the strong denoising capacity of those models. Nevertheless, such two-stage methods involve longer training pipelines and are more vulnerable to adversarial attacks, as will be demonstrated in the next section. Additional experimental results on label noise are shown in Appendix E. 

\subsection{Robustness Against Adversarial Perturbations}
In this section, to further validate the ``Minimal-Sufficient" property in MSC of the proposed \method method, we apply perturbations to the input data $\mathcal{D}$. The perturbed data is then fed into models trained under noisy label settings. This setup enables a comprehensive evaluation of model robustness against diverse noise, including inputs and labels.

\textbf{Results.} The results demonstrate that the \method method exhibits strong robustness against adversarial attacks, significantly outperforming other approaches, as shown in Table \ref{tab:cifar10n_attack}. Notably, two-stage methods suffer a substantial performance drop under attack due to the increased number of vulnerable components, further highlighting their limitations.

\begin{figure}[!t]
    \centering
    \includegraphics[width=0.47\textwidth]{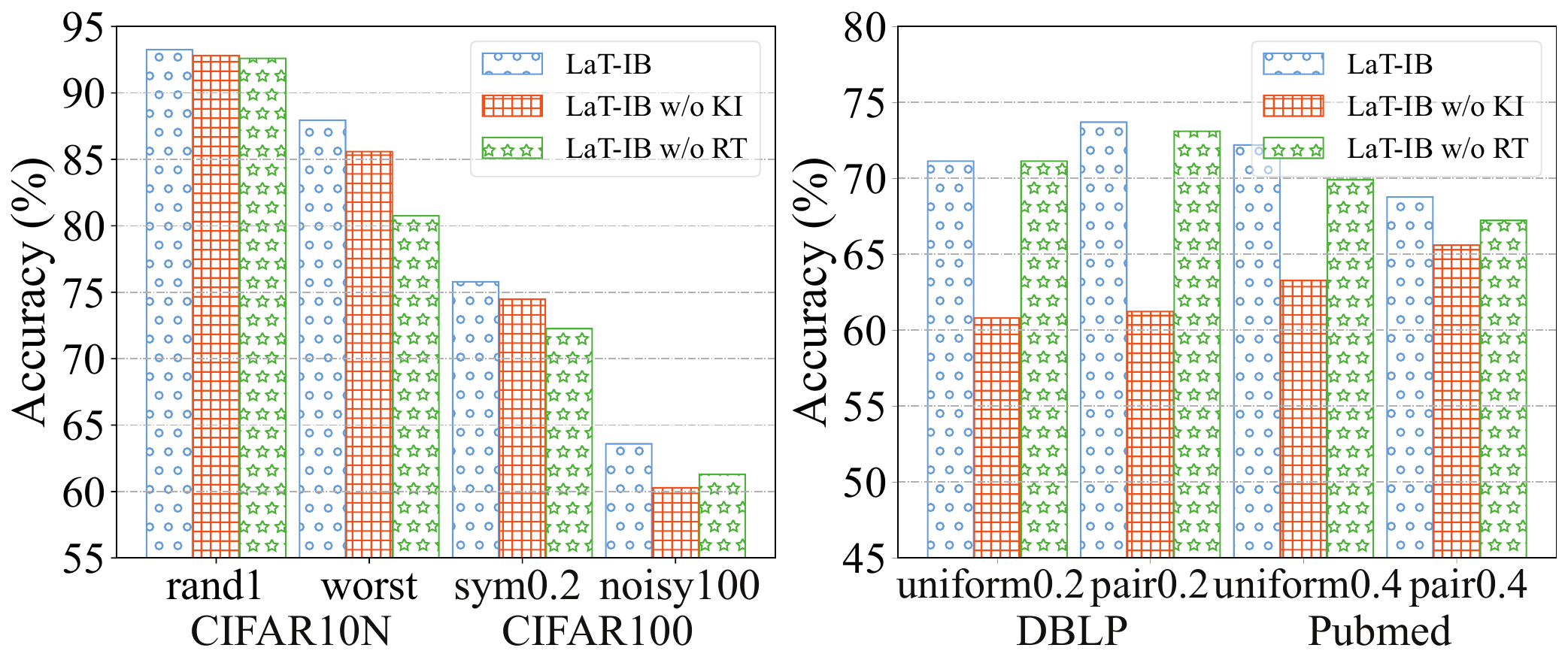}
    \caption{Ablation study.}
\end{figure}

\subsection{Ablation Study}
In this section, we analyze the effectiveness of different training stages in the \method framework.
To investigate the role of each phase in enhancing model robustness, we design two ablated variants:
\begin{itemize}
    \item \textbf{\method (\textit{w/o KI})}: We remove the Knowledge Injection period, thus $\max(I(Y_n;S), I(Y_c;T)) \leq K$ is not satisfied. weakening the ability to map $S \to Y_c$ and $T\to Y_n$.
    \item \textbf{\method (\textit{w/o RT})}: We remove the Robust Training period, meaning no further enhancement is applied to the representations from $S, T$. The \method model can only gain partial information from the three subsets. 
\end{itemize}

Note that we do not design an ablation variant without the Warmup period, as it is essential for establishing basic classification capability and stable later training.

\textbf{Results.} Overall, the full \method method achieves the best performance under all noisy label settings as shown in Figure 3, demonstrating the importance of different periods in the framework. For image classification tasks (with larger samples), the Robust Training stage is particularly critical, while for graph-based tasks (with fewer samples), the Knowledge Injection stage proves more influential. These findings highlight the necessity of each training stage in achieving robust representations under noisy supervision.

\subsection{Hyperparameter Sensitivity Analysis}
\label{hyper_beta}
We analyze the sensitivity of the model to the hyperparameter $\beta$, $\gamma$ and $\delta$.
The coefficient $\beta$ controls the feature compression term $I(\mathcal{D}; S, T)$, which encourages the model to learn noise invariant features. The coefficient $\gamma$ controls the feature separation term $I(S;T|Y)$, which encourages the $\text{encoder}_{S,T}$ to capture clean and noisy representations respectively. $\delta$ regulates how much information the $\text{encoder}_{S,T}$ learns during the Knowledge Injection phase.

\textbf{Results.} We observe that a large $\beta$ can dominate training and cause collapse as shown in Figure 4, indicating that $I(\mathcal{D}; S, T)$ partially limits the model’s expressiveness. However, our method is more tolerant to $\beta$ than the original VIB, which fails to train on CIFAR-10 with 50\% symmetric noise at $\beta = 0.01$. In contrast, our model performs better as $\beta$ decreases because the input compression level is reduced.

We also observe that the model's sensitivity to $\gamma$ varies across noisy settings, highlighting the importance of the separation term $I(S;T|Y)$ under different types of noise.

For $\delta$, a too-small $\delta$ limits the encoder’s training data exposure, while a too-large $\delta$ causes the encoders to converge, reducing their ability to separate clean and noisy information.
More detailed results in Appendix E.6.

\begin{figure}[!t]
    \centering
    \includegraphics[width=0.46\textwidth]{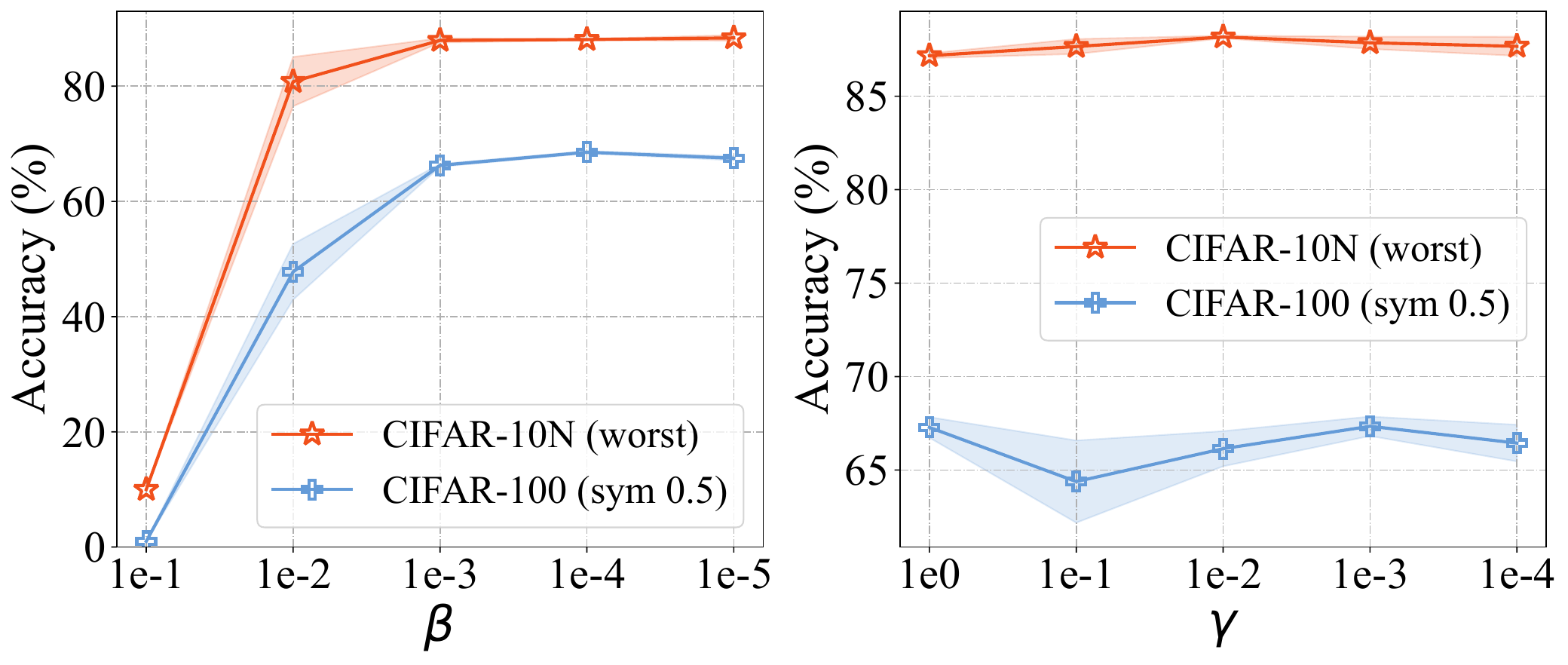}
    \caption{The influence of $\beta$ and $\gamma$.}
\end{figure}

\section{Conclusion}
In this work, we propose \textbf{\method}, a novel yet principled IB framework that enables robust representation learning under label noise while preserving the principle of learning minimally sufficient representations.
We disentangle features into representations related to clean and noisy label spaces, and theoretically demonstrate the noise-separating effect of our method through upper and lower bounds analysis. Furthermore, we design a  three-phase training framework comprising \periods, that facilitates the extraction of \textit{``Minimal-Sufficient-Clean"} representations. Extensive experiments across diverse noisy environments validate the superior performance of \method compared to existing IB-based methods, highlighting its potential to efficiently advance the practical application of IB theory in real-world learning scenarios with label noise.

\section*{Acknowledgments}
The corresponding author is Qingyun Sun. This work is supported by NSFC under grants No.62427808 and No.62225202, and by the Fundamental Research Funds for the Central Universities. We extend our sincere thanks to all reviewers for their valuable efforts.

\bibliography{aaai2026}

@article{GIB,
  title={Graph information bottleneck},
  author={Wu, Tailin and Ren, Hongyu and Li, Pan and Leskovec, Jure},
  journal={Advances in Neural Information Processing Systems},
  volume={33},
  pages={20437--20448},
  year={2020}
}

@article{IB,
  title={The information bottleneck method},
  author={Tishby, Naftali and Pereira, Fernando C and Bialek, William},
  journal={arXiv preprint physics/0004057},
  year={2000}
}

@inproceedings{VIB,
  title={Deep Variational Information Bottleneck},
  author={Alemi, Alexander A and Fischer, Ian and Dillon, Joshua V and Murphy, Kevin},
  booktitle={International Conference on Learning Representations},
  year={2017}
}

@inproceedings{CIFARN,
  title={Learning with Noisy Labels Revisited: A Study Using Real-World Human Annotations},
  author={Wei, Jiaheng and Zhu, Zhaowei and Cheng, Hao and Liu, Tongliang and Niu, Gang and Liu, Yang},
  booktitle={International Conference on Learning Representations},
  year={2022}
}

@article{ELR,
  title={Early-learning regularization prevents memorization of noisy labels},
  author={Liu, Sheng and Niles-Weed, Jonathan and Razavian, Narges and Fernandez-Granda, Carlos},
  journal={Advances in neural information processing systems},
  volume={33},
  pages={20331--20342},
  year={2020}
}

@article{LN-survey,
  title={Learning from noisy labels with deep neural networks: A survey},
  author={Song, Hwanjun and Kim, Minseok and Park, Dongmin and Shin, Yooju and Lee, Jae-Gil},
  journal={IEEE transactions on neural networks and learning systems},
  volume={34},
  number={11},
  pages={8135--8153},
  year={2022},
  publisher={IEEE}
}

@article{IB-survey,
  title={A survey on information bottleneck},
  author={Hu, Shizhe and Lou, Zhengzheng and Yan, Xiaoqiang and Ye, Yangdong},
  journal={IEEE Transactions on Pattern Analysis and Machine Intelligence},
  year={2024},
  publisher={IEEE}
}

@article{x_robust1,
  title={Learning and generalization with the information bottleneck},
  author={Shamir, Ohad and Sabato, Sivan and Tishby, Naftali},
  journal={Theoretical Computer Science},
  volume={411},
  number={29-30},
  pages={2696--2711},
  year={2010},
  publisher={Elsevier}
}

@article{link-GIB,
  title={Combating bilateral edge noise for robust link prediction},
  author={Zhou, Zhanke and Yao, Jiangchao and Liu, Jiaxu and Guo, Xiawei and Yao, Quanming and He, Li and Wang, Liang and Zheng, Bo and Han, Bo},
  journal={Advances in Neural Information Processing Systems},
  volume={36},
  pages={21368--21414},
  year={2023}
}

@inproceedings{DisenIB,
  title={Disentangled information bottleneck},
  author={Pan, Ziqi and Niu, Li and Zhang, Jianfu and Zhang, Liqing},
  booktitle={Proceedings of the AAAI Conference on Artificial Intelligence},
  volume={35},
  pages={9285--9293},
  year={2021}
}

@inproceedings{DGIB,
  title={Dynamic graph information bottleneck},
  author={Yuan, Haonan and Sun, Qingyun and Fu, Xingcheng and Ji, Cheng and Li, Jianxin},
  booktitle={Proceedings of the ACM Web Conference 2024},
  pages={469--480},
  year={2024}
}

@article{GCE,
  title={Generalized cross entropy loss for training deep neural networks with noisy labels},
  author={Zhang, Zhilu and Sabuncu, Mert},
  journal={Advances in neural information processing systems},
  volume={31},
  year={2018}
}

@inproceedings{sample-selection,
  title={Adaptive sample selection for robust learning under label noise},
  author={Patel, Deep and Sastry, PS},
  booktitle={Proceedings of the IEEE/CVF Winter Conference on Applications of Computer Vision},
  pages={3932--3942},
  year={2023}
}

@inproceedings{promix,
  title={ProMix: combating label noise via maximizing clean sample utility},
  author={Xiao, Ruixuan and Dong, Yiwen and Wang, Haobo and Feng, Lei and Wu, Runze and Chen, Gang and Zhao, Junbo},
  booktitle={Proceedings of the Thirty-Second International Joint Conference on Artificial Intelligence},
  pages={4442--4450},
  year={2023}
}

@article{noisyGL,
  title={NoisyGL: A Comprehensive Benchmark for Graph Neural Networks under Label Noise},
  author={Wang, Zhonghao and Sun, Danyu and Zhou, Sheng and Wang, Haobo and Fan, Jiapei and Huang, Longtao and Bu, Jiajun},
  journal={arXiv preprint arXiv:2406.04299},
  year={2024}
}

@inproceedings{arpit2017a,
  title={A closer look at memorization in deep networks},
  author={Arpit, Devansh and Jastrzebski, Stanis{\l}aw and Ballas, Nicolas and Krueger, David and Bengio, Emmanuel and Kanwal, Maxinder S and Maharaj, Tegan and Fischer, Asja and Courville, Aaron and Bengio, Yoshua and others},
  booktitle={International conference on machine learning},
  pages={233--242},
  year={2017},
  organization={PMLR}
}

@article{song2019does,
  title={How does early stopping help generalization against label noise?},
  author={Song, Hwanjun and Kim, Minseok and Park, Dongmin and Lee, Jae-Gil},
  journal={arXiv preprint arXiv:1911.08059},
  year={2019}
}

@article{sugiyama2012density,
  title={Density-ratio matching under the bregman divergence: a unified framework of density-ratio estimation},
  author={Sugiyama, Masashi and Suzuki, Taiji and Kanamori, Takafumi},
  journal={Annals of the Institute of Statistical Mathematics},
  volume={64},
  pages={1009--1044},
  year={2012},
  publisher={Springer}
}

@inproceedings{fgsm,
  author       = {Ian J. Goodfellow and
                  Jonathon Shlens and
                  Christian Szegedy},
  editor       = {Yoshua Bengio and
                  Yann LeCun},
  title        = {Explaining and Harnessing Adversarial Examples},
  booktitle    = {3rd International Conference on Learning Representations, {ICLR} 2015,
                  San Diego, CA, USA, May 7-9, 2015, Conference Track Proceedings},
  year         = {2015}
}

@article{fano,
  title={Class notes for course 6.574: Transmission of information},
  author={Fano, RM},
  journal={Lecture Notes},
  year={1952}
}

@inproceedings{animal-10N,
  title={Selfie: Refurbishing unclean samples for robust deep learning},
  author={Song, Hwanjun and Kim, Minseok and Lee, Jae-Gil},
  booktitle={International conference on machine learning},
  pages={5907--5915},
  year={2019},
  organization={PMLR}
}

@article{cifar,
  title={Learning multiple layers of features from tiny images},
  author={Krizhevsky, Alex and Hinton, Geoffrey and others},
  year={2009},
  publisher={Toronto, ON, Canada}
}

@article{G_data,
  title={Collective classification in network data},
  author={Sen, Prithviraj and Namata, Galileo and Bilgic, Mustafa and Getoor, Lise and Galligher, Brian and Eliassi-Rad, Tina},
  journal={AI magazine},
  volume={29},
  number={3},
  pages={93--93},
  year={2008}
}

@inproceedings{dblp,
  title={Tri-party deep network representation},
  author={Pan, Shirui and Wu, Jia and Zhu, Xingquan and Zhang, Chengqi and Wang, Yang},
  booktitle={International Joint Conference on Artificial Intelligence 2016},
  pages={1895--1901},
  year={2016},
  organization={Association for the Advancement of Artificial Intelligence (AAAI)}
}

@article{fmix,
  title={Fmix: Enhancing mixed sample data augmentation},
  author={Harris, Ethan and Marcu, Antonia and Painter, Matthew and Niranjan, Mahesan and Pr{\"u}gel-Bennett, Adam and Hare, Jonathon},
  journal={arXiv preprint arXiv:2002.12047},
  year={2020}
}

@article{menonlong,
  title={Long-tail learning via logit adjustment},
  author={Menon, Aditya Krishna and Jayasumana, Sadeep and Rawat, Ankit Singh and Jain, Himanshu and Veit, Andreas and Kumar, Sanjiv},
  journal={arXiv preprint arXiv:2007.07314},
  year={2020}
}

@inproceedings{SIB,
  title={Structured IB: Improving Information Bottleneck with Structured Feature Learning},
  author={Yang, Hanzhe and Wu, Youlong and Wen, Dingzhu and Zhou, Yong and Shi, Yuanming},
  booktitle={Proceedings of the AAAI Conference on Artificial Intelligence},
  volume={39},
  pages={21922--21928},
  year={2025}
}

@article{DT-JSCC,
  title={Robust information bottleneck for task-oriented communication with digital modulation},
  author={Xie, Songjie and Ma, Shuai and Ding, Ming and Shi, Yuanming and Tang, Mingjian and Wu, Youlong},
  journal={IEEE Journal on Selected Areas in Communications},
  volume={41},
  number={8},
  pages={2577--2591},
  year={2023},
  publisher={IEEE}
}

@inproceedings{JoCoR,
  title={Combating noisy labels by agreement: A joint training method with co-regularization},
  author={Wei, Hongxin and Feng, Lei and Chen, Xiangyu and An, Bo},
  booktitle={Proceedings of the IEEE/CVF conference on computer vision and pattern recognition},
  pages={13726--13735},
  year={2020}
}

@inproceedings{SCE,
  title={Symmetric cross entropy for robust learning with noisy labels},
  author={Wang, Yisen and Ma, Xingjun and Chen, Zaiyi and Luo, Yuan and Yi, Jinfeng and Bailey, James},
  booktitle={Proceedings of the IEEE/CVF international conference on computer vision},
  pages={322--330},
  year={2019}
}

@article{IS-GIB,
  title={Individual and structural graph information bottlenecks for out-of-distribution generalization},
  author={Yang, Ling and Zheng, Jiayi and Wang, Heyuan and Liu, Zhongyi and Huang, Zhilin and Hong, Shenda and Zhang, Wentao and Cui, Bin},
  journal={IEEE Transactions on Knowledge and Data Engineering},
  volume={36},
  number={2},
  pages={682--693},
  year={2023},
  publisher={IEEE}
}

@inproceedings{CurvGIB,
  title={Discrete curvature graph information bottleneck},
  author={Fu, Xingcheng and Wang, Jian and Gao, Yisen and Sun, Qingyun and Yuan, Haonan and Li, Jianxin and Li, Xianxian},
  booktitle={Proceedings of the AAAI Conference on Artificial Intelligence},
  volume={39},
  pages={16666--16673},
  year={2025}
}

@inproceedings{RNCGLN,
  title={Robust node classification on graph data with graph and label noise},
  author={Zhu, Yonghua and Feng, Lei and Deng, Zhenyun and Chen, Yang and Amor, Robert and Witbrock, Michael},
  booktitle={Proceedings of the AAAI conference on artificial intelligence},
  volume={38},
  pages={17220--17227},
  year={2024}
}

@inproceedings{cgnn,
  title={Learning on graphs under label noise},
  author={Yuan, Jingyang and Luo, Xiao and Qin, Yifang and Zhao, Yusheng and Ju, Wei and Zhang, Ming},
  booktitle={ICASSP 2023-2023 IEEE International Conference on Acoustics, Speech and Signal Processing (ICASSP)},
  pages={1--5},
  year={2023},
  organization={IEEE}
}

@article{NIB,
  title={Nonlinear information bottleneck},
  author={Kolchinsky, Artemy and Tracey, Brendan D and Wolpert, David H},
  journal={Entropy},
  volume={21},
  number={12},
  pages={1181},
  year={2019},
  publisher={MDPI}
}

@inproceedings{mixup,
  title={mixup: Beyond Empirical Risk Minimization},
  author={Zhang, Hongyi and Cisse, Moustapha and Dauphin, Yann N and Lopez-Paz, David},
  booktitle={International Conference on Learning Representations},
  year={2018}
}

@article{pensia2020extracting,
  title={Extracting robust and accurate features via a robust information bottleneck},
  author={Pensia, Ankit and Jog, Varun and Loh, Po-Ling},
  journal={IEEE Journal on Selected Areas in Information Theory},
  volume={1},
  number={1},
  pages={131--144},
  year={2020},
  publisher={IEEE}
}

@article{li2025simplified,
  title={Simplified PCNet with robustness},
  author={Li, Bingheng and Xie, Xuanting and Lei, Haoxiang and Fang, Ruiyi and Kang, Zhao},
  journal={Neural Networks},
  volume={184},
  pages={107099},
  year={2025},
  publisher={Elsevier}
}

\clearpage
\newpage
\appendix
\counterwithin{table}{section}
\counterwithin{figure}{section}
\counterwithin{equation}{section}
\counterwithin{equation}{section}
\counterwithin{algorithm}{section}

\section{Code Source}
\label{appendix_code}
The code for \method is publicly available at: https://github.com/RingBDStack/LaT-IB.

\section{Algorithms and Complexity Analysis}
\label{appendix_alg}
\subsection{The Overall Process of \method}
The overall training pipeline of \method is outlined in Algorithm~\ref{alg_overall}.

\begin{algorithm}[ht]
\caption{The Overall Process of \method}
\label{alg_overall}
\textbf{Input}: Input data $\mathcal{D}$ with label $Y$, epoch number of each period $E_{Warmup}, E_{Injection}, E_{Robust}$, overall epoch number $E$ \\
\textbf{Output}: Predicted label $\hat Y$

\begin{algorithmic}[1]
\STATE Parameter initialization
\FOR{epoch $i = 1, \dots, E$}
    \STATE Encode inputs to obtain $(\mu_S, \sigma_S), (\mu_T, \sigma_T)$
    \STATE Decode the reparameterized embedding to generate predictions $(\hat{y}_S, \hat{y}_T)$
    \IF{$i \leq E_{Warmup}$}
        \STATE  Optimize the $\text{encoder}_S$ with Eq. (13)
    \ENDIF
    \IF{$E_{Warmup} < i \leq E_{Injection}$}
        \STATE  Obtain Clean/Noise/Uncertain Set using InfoJS selector
        \STATE  Optimize the \method model with Eq. (15) and (16)
    \ENDIF
    \IF{$i > E_{Injection}$}
        \STATE  Optimize the \method model with Eq. (18)
        \STATE  Optimize the discriminator with Eq. (19)
    \ENDIF
\ENDFOR
\end{algorithmic}
\end{algorithm}

\subsection{InfoJS Selector Algorithms Details}
Algorithm \ref{alg_sel} implements a sample selection scheme based on MI and JS divergence.

\begin{algorithm}[ht]
\caption{InfoJS Selector}
\label{alg_sel}
\textbf{Input}: Trained model $f_\theta$, data $(\mathcal{D},Y)$, selection ratio $\delta$ \\
\textbf{Output}: Binary selection mask: $\text{mask}_S$, $\text{mask}_T$ 

\begin{algorithmic}[1]
\STATE Run $f_\theta$ to obtain $(\mu_S, \sigma_S), (\mu_T, \sigma_T)$ and predictions $(\hat{y}_S, \hat{y}_T)$
\STATE Compute $I(S;Y)$: $\ell_S \gets \mathcal L_{CE}(\hat{y}_S, y)$ \COMMENT{Eq. (10)}
\STATE Compute JS divergence: $JS \gets D_{JS}(\mu_S, \sigma_S, \mu_T, \sigma_T)$
\STATE Initialize $\text{mask}_S$, $\text{mask}_T$ 
\FOR{each class $j = 1, \dots, C$}
    \STATE $\mathcal{I}_j \gets$ indices of class-$j$ samples in $\mathcal{M}$
    \STATE $k_j \gets \max(1, \min(\lceil \delta \cdot |\mathcal{I}_j| \rceil, |\mathcal{D}|/C))$
    \STATE Select top-$k_j$ and bottom-$k_j$ samples by $\ell_S$ and update $\text{mask}_S$, $\text{mask}_T$ %
    \STATE Select top-$k_j$ and bottom-$k_j$ samples by $JS$ and update $\text{mask}_S$, $\text{mask}_T$ %
\ENDFOR
\RETURN $\text{mask}_S$, $\text{mask}_T$ \COMMENT{$\text{mask}_S$ for clean set; $\text{mask}_T$ for noise set}
\end{algorithmic}
\end{algorithm}

\subsection{$\mathcal{L}_{ConCE}$ Algorithms Details}
Proposition 4.1 states:
\begin{equation}
    -I(Y;S,T) \leq -\max(I(Y;S), I(Y;T)).
\end{equation}

Further reformulated as a cross-entropy loss, this provides an implementation method for $I(Y;S,T)$:
\begin{equation}
\label{alg_from}
\begin{aligned}
    -I(Y;S,T) \leq \sum_i\min(\mathcal L(\hat y_{S,i}, y_i), \mathcal L(\hat y_{T,i}, y_i)).
\end{aligned}
\end{equation}

In the loss function design, we tailor the loss for two scenarios: image classification (many samples) and node classification (few samples). 

\subsubsection{Graph Tasks.}
First, we present the loss function design for the node classification task under the few-sample condition, as it is closest to Equation \eqref{alg_from}. Since the $\min$ function is not continuous and may hinder model learning, we instead use the smooth approximation:
\begin{equation}
\label{eq_min}
    f(a,b)=\frac{1}{\beta} \cdot \log \left( \exp(-\beta \cdot a) + \exp(-\beta \cdot b) \right).
\end{equation}
Notably, as $b\to +\infty$, this expression becomes equivalent to $\min(a,b)$.

\begin{algorithm}[ht]
\caption{$\mathcal L_{ConCE}$ for node classification}
\label{alg1}
\textbf{Input}: Predictions $prob_1$, $prob_2$; Labels $y$; Threshold $\beta$ \\
\textbf{Output}: Final loss value $\mathcal L_{ConCE}$

\begin{algorithmic}[1]
\STATE $a \gets \mathcal{L}_{CE}(prob_1, y)$ \COMMENT{Per-sample $\mathcal L_{CE}$}
\STATE $b \gets \mathcal{L}_{CE}(prob_2, y)$
\STATE $L \gets \frac{1}{\beta} \cdot \log \left( \exp(-\beta \cdot a) + \exp(-\beta \cdot b) \right)$ \COMMENT{Eq. \eqref{eq_min}}
\RETURN $\text{Mean}(L)$ 
\end{algorithmic}
\end{algorithm}

\subsubsection{Image Tasks.}
We then directly apply Algorithm~\ref{alg1} to the image classification task and observe that it causes lazy training, where the model tends to produce identical outputs. We hypothesize that this behavior arises because: as the number of training samples increases, the gradient contributions from noisy samples become diluted. Meanwhile, Algorithm~\ref{alg1} imposes a strict separation, causing each encoder to receive only a limited portion of the input data. As a result, when the predictions from the two encoders become overly similar, the model lacks sufficient supervision signals, potentially leading to training collapse.

To address this issue, we train both encoders on the samples where the two prediction heads produce consistent outputs, rather than assigning them to a single encoder. This significantly increases the amount of usable information for each encoder. Moreover, we relax the condition in Equation~\ref{alg_from} to:
\begin{equation}
\begin{aligned}
    -I(Y;S,T) \leq \sum_i\min(\mathcal L(\hat y_{S,i}, y_i), \mathcal L(\hat y_{T,i}, y_i)) \\
     \leq \sum_i \max(\mathcal L(\hat y_{S,i}, y_i), \mathcal L(\hat y_{T,i}, y_i))
\end{aligned}
\end{equation}

The final loss function is shown in Algorithm~\ref{alg2}.

\begin{algorithm}[ht]
\caption{$\mathcal L_{ConCE}$ for image classification}
\label{alg2}
\textbf{Input}: Predictions $prob_1$, $prob_2$; Labels $y$; Threshold $t$ \\
\textbf{Output}: Final loss value $\mathcal L_{ConCE}$

\begin{algorithmic}[1]
\STATE Compute $a \gets \mathcal{L}_{CE}(prob_1.\text{copy()}, y)$ 
\STATE Compute $b \gets \mathcal{L}_{CE}(prob_2.\text{copy()}, y)$
\STATE $pred_1 \gets \arg\max(prob_1.\text{copy()})$ \COMMENT{Per-sample $\mathcal L_{CE}$}
\STATE $pred_2 \gets \arg\max(prob_2.\text{copy()})$
\STATE $mask_a \gets (a < b)$
\STATE $mask_b \gets (a > b) \lor (pred_1 = pred_2)$ \COMMENT{The second part prevents $T$ from learning nothing.}
\IF{$\sum (pred_1 = pred2) /|y| > t$}
    \STATE $mask_a \gets (a < b) \lor (pred_1 = pred2)$
    \STATE $mask_b \gets (a > b)$
\ENDIF
\STATE Initialize $y_1 \gets y$, $y_2 \gets y$
\STATE Replace $y_1[mask_a] \gets pred_1[mask_a]$ \COMMENT{Focus on clean samples}
\STATE Replace $y_2[mask_b] \gets pred_2[mask_b]$ \COMMENT{Focus on noise samples}
\STATE $loss_1 \gets \mathcal{L}_{CE}(prob_1, y_1)$
\STATE $loss_2 \gets \mathcal{L}_{CE}(prob_2, y_2)$
\RETURN $(loss_1 + loss_2)/2$
\end{algorithmic}
\end{algorithm}

\subsection{Time Consumption}
We adopt the constant definitions from Section 3.
For Algorithm~\ref{alg_sel}, the $\mathcal L_{CE}$ has a time complexity of $\mathcal{O}(NC)$. The JS divergence, computed between all sample pairs, requires $\mathcal{O}(N^2k)$ operations, where $k$ is the dimensionality of the Gaussian embeddings. Assuming class-balanced data, the per-class selection of top-$k$ and bottom-$k$ samples involves sorting subsets of size approximately $N/C$, giving a total sorting complexity of $\mathcal{O}(N \log(N/C))$. Overall, the dominant term is $\mathcal{O}(N^2k)$, since typically $C \ll N$, making the pairwise divergence computation the main computational bottleneck.

Algorithms \ref{alg1} and \ref{alg2} perform a finite number of cross-entropy loss computations. Therefore, their overall time complexity is equivalent to that of the cross-entropy loss, which is $\mathcal{O}(NC)$.

To further assess the efficiency of different methods, we report the training time consumption of several top-performing approaches.

\subsubsection{Image Classification.}
We selected three baselines including VIB, (ELR+)+VIB and (Promix\footnote{For the Promix method, strictly following the original experimental setup results in a runtime exceeding 24 hours, leading to out-of-time (OOT) termination. To address this, we reduced the number of training epochs, with only a minor performance drop, to ensure timely completion.})+VIB as well as compared their time consumption under the condition of achieving a performance like Table 2 on CIFAR-10N dataset, as shown in the Table~\ref{tab:time_consumption_IC}. The results demonstrate that although our model incurs higher computation than VIB, it achieves better performance within fewer epochs, and overall outperforms two-stage models in terms of efficiency.

\begin{table}[ht]
\centering
\hspace{-0.17cm}
\resizebox{0.48\textwidth}{!}{
\begin{tabular}{lccc|c}
\toprule
\textbf{Method} & VIB & (ELR+)+VIB & (Promix)+VIB &  \method \\
\midrule
\textbf{Time (h)} & 1.6 & 4.9+1.6 & 12.9+1.6 & 5.18\\ 
\textbf{Epoch} & 100 & 200+100 & 300+100 & 200\\
\textbf{Per Epoch (min)} & 0.94 & 1.30 & 2.18 & 1.55 \\
\bottomrule
\end{tabular}
}
\caption{Comparison of training time consumption (IC)}
\label{tab:time_consumption_IC}
\end{table}

\subsubsection{Node Classification.}
We selected GIB, GIB ($\mathcal{L}_{GCE}$) and RNCGLN+GIB as baselines. Table \ref{tab:time_consumption_NC} presents the results on the DBLP dataset, supporting similar conclusions as in the image classification task. Notably, although RNCGLN achieves strong performance, it incurs significant time overhead, as also reported in NoisyGL~\cite{noisyGL}.

\begin{table}[ht]
\centering
\hspace{-0.17cm}
\resizebox{0.48\textwidth}{!}{
\begin{tabular}{lccc|c}
\toprule
\textbf{Method} & GIB & GIB ($\mathcal{L}_{GCE}$) & RNCGLN+GIB &  \method \\
\midrule
\textbf{Time (s)} & 37.8 & 37.8 & 3445.4+37.8 & 55.9\\ 
\textbf{Epoch} & 100 & 100 & 500+100 & 100\\
\textbf{Per Epoch (s)} & 0.38 & 0.38 & 5.80 & 0.56 \\
\bottomrule
\end{tabular}
}
\caption{Comparison of training time consumption (NC)}
\label{tab:time_consumption_NC}
\end{table}

We evaluate the time cost at each period of the pipeline. The results are shown in Figure~\ref{fig:time_cinsumption}. The results indicate that for image tasks with a \textbf{large} number of samples, Knowledge Injection consumes the most time, which aligns with our theoretical time complexity analysis. In contrast, for graph tasks with \textbf{fewer} samples, the time cost of Knowledge Injection is slightly lower than that of Robust Training. We speculate that this is due to lower actual time complexity than the theoretical value, possibly resulting from low-level computational optimizations, and the relatively small sample size helps offset part of the time overhead.

\begin{figure}[ht]
\centering
\begin{subfigure}[b]{0.46\linewidth}
    \includegraphics[width=\linewidth]{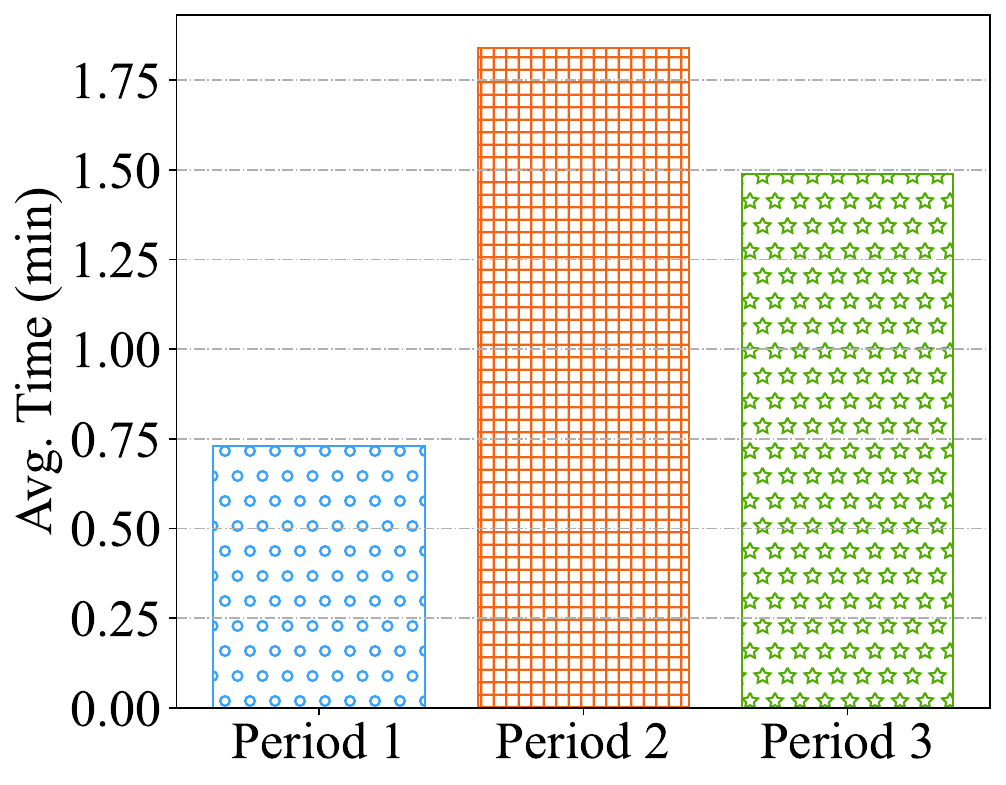}
    \caption{CIFAR-10N (worst)}
    \label{fig:cifar_time}
\end{subfigure}
\begin{subfigure}[b]{0.45\linewidth}
    \includegraphics[width=\linewidth]{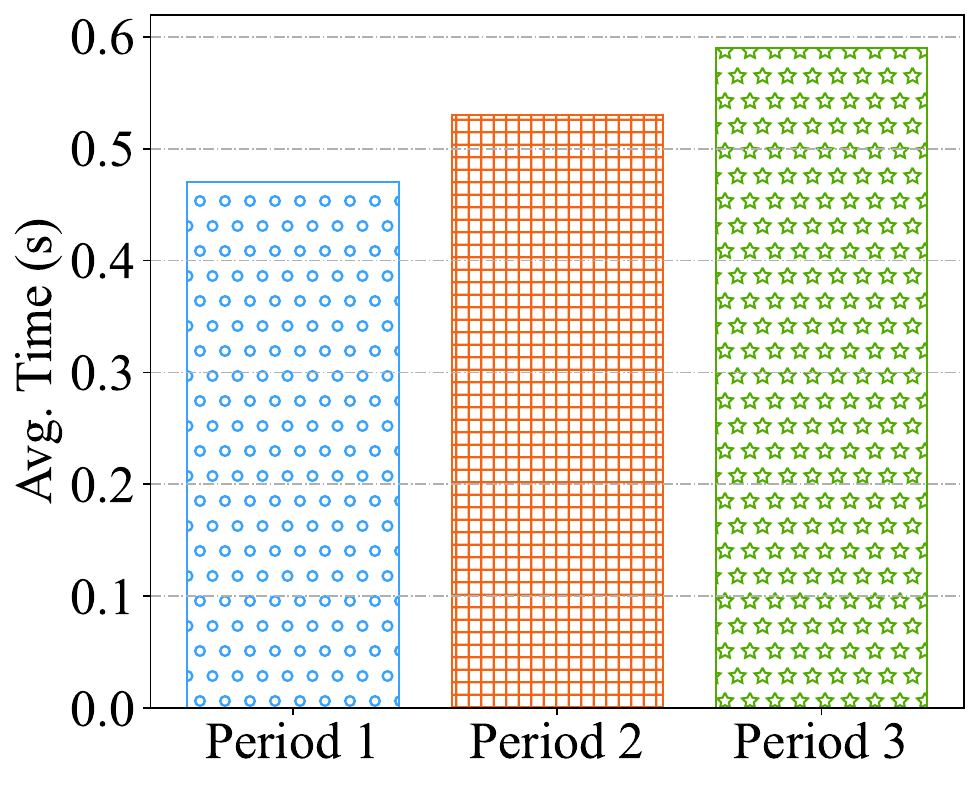}
    \caption{DBLP}
    \label{fig:dblp_time}
\end{subfigure}
\caption{Time Consumption Analysis}
\label{fig:time_cinsumption}
\end{figure}

\section{Proofs}
\label{appendix_proof}
In this section, we present the relevant proofs from the paper. We first restate the statements to be proved, followed by detailed proofs.

\subsection{The proof of Theorem 1.1}
\label{appendix_theorem_proof}
\begin{theorem}[Cumulative Degradation]
    In the two-stage approach, $f_1$ is used to modify the labels $Y'=f_1(\mathcal{D})$, and $f_2$ is responsible for extracting valid information from $\mathcal{D}$ while approximating the prediction result to $f_1(\mathcal{D})$ . For one-stage model $g(\mathcal{D})$, it extracts the relevant information while removing noise. If the denoising abilities of $f_1$ and $g$ are the same, the following inequality holds:
    \begin{equation}
        P(f_2(\mathcal{D})\neq g(\mathcal{D})) \geq \frac{H(Y'|\mathcal{D})-1}{\log(|\mathcal{Y}| - 1)},
    \end{equation}
    where $\mathcal{Y}$ denotes the support of $Y$, and $|\mathcal{Y}|$ denotes the number of elements in  $\mathcal{Y}$. The two models perform identically iff $f_2$ achieves the error lower bound and $H(Y'|\mathcal{D}) = 0$.
\end{theorem}

\begin{proof}
    Without loss of generality, we illustrate this using the image classification task, where $\mathcal{D}=X$.

    We begin by introducing \textbf{Fano's inequality}~\cite{fano}: Let the discrete random variables $X$ and $Y$ denote the input and output messages, respectively, with joint distribution $P(x, y)$. Let $e$ represent the event of an error, i.e., $X \neq \tilde{X}$, where $\tilde{X} = f(Y)$ is an estimate of $X$. Then \textbf{Fano's inequality} states:
    \begin{equation}
        H(X | Y) \leq H_b(e) + P(e) \log(|\mathcal{X}| - 1),
    \end{equation}
    where $\mathcal{X}$ denotes the support set of the random variable $X$, and $|\mathcal{X}|$ is its cardinality (i.e., the number of elements in $\mathcal{X}$). Here, $H(X | Y) = -\sum_{i,j} P(x_i, y_j) \log P(x_i | y_j)$ is the conditional entropy, $P(e) = P(X \neq \hat{X})$ is the probability of a communication error, and $H_b(e) = -P(e) \log P(e) - (1 - P(e)) \log (1 - P(e))$ is the binary entropy.

    \textbf{Part 1:}
    Consider the second stage $f_2$ of the two-stage model, whose input $Y'$ serves as a reference for the final model output $f_2(x)$. According to the inequality above, we have:
    \begin{equation}
    \label{usage}
        H(Y' | X) \leq H_b(e) + P(e) \log(|\mathcal{Y}| - 1),
    \end{equation}
    where $P(e) := P(f_2(x) \neq f_1(x))$. Since the binary entropy has an upper bound of \textnormal{1}, the lower bound of the error probability is:
    \begin{equation}
    \begin{aligned}
        H(Y' | X) &\leq H_b(e) + P(e) \log(|\mathcal{Y}| - 1) \\
        &\leq 1 + P(e) \log(|\mathcal{Y}| - 1),
    \end{aligned}
    \end{equation}
    Given the assumption that $f_1$ and $g$ have the same denoising capability, we can, without loss of generality, assume $f_1(x) = g(x)$, thus:
    \begin{equation}
        P(f_2(x) \neq g(x)) = P(e) \geq \frac{H(Y' | X) - 1}{\log(|\mathcal{Y}| - 1)}.
    \end{equation}

    \textbf{Part 2:}
    When the two models behave identically, i.e., $P(f_2(x) \neq g(x)) = 0$, from inequality \textnormal{\ref{usage}} we know that $H(Y' | X) \leq 0$. Since entropy is non-negative, it follows that $H(Y' | X) = 0$. At this point, $f_2$ has optimized its classification error to the theoretical lower bound.

    When $H(Y' | X) = 0$, it is easy to see that $P(f_2(x) \neq g(x)) = 0$ satisfies inequality \textnormal{\ref{usage}}. Since $P(f_2(x) \neq g(x)) \geq 0$, and \textnormal{0} is the lower bound of the error rate, if $f_2$ reaches this bound, the performance of the two models is exactly the same.
\end{proof}

\subsection{The proof of Lemma 4.1 and 4.2}
\label{appendix_lemma_proof}
\begin{lemma}[Nuisance Invariance]
    Taking the part of $\mathcal{D}$ that does not contribute to $Y$ as $\mathcal{D}_n$ ($D_n$ is independent of $Y$), and considering the Markov chain $(Y,\mathcal{D}_n) \to \mathcal{D} \to (S,T)$, the following inequality holds:
    \begin{equation}
        I(\mathcal{D}_n;S,T)\leq -I(Y;S,T)+I(\mathcal{D};S,T).
    \end{equation} 
\end{lemma}

\begin{proof}
    We proof this lemma in three steps.

    \textbf{Step 1:}
    According to the property of the Markov chain, we have:
    \begin{equation}
        I(S, T; \mathcal{D}_n, Y | \mathcal{D})=0.
    \end{equation}
    And because: 
    \begin{equation}
    \begin{aligned}
        &I(S, T; \mathcal{D}_n, Y | \mathcal{D})\\
        =&\underbrace{I(S, T; \mathcal{D}_n | \mathcal{D})}_{\geq 0} + \underbrace{I(S, T; Y |  \mathcal{D}_n, \mathcal{D})}_{\geq 0}\\
        =&0.
    \end{aligned}
    \end{equation}
    So we have:
    \begin{equation}
        I(S, T; \mathcal{D}_n | \mathcal{D}) =I(S, T; Y |  \mathcal{D}_n, \mathcal{D}) = 0
    \end{equation}
    
    By expanding $I(S, T; \mathcal{D}, \mathcal{D}_n)$, we obtain:
    \begin{equation}
    \label{step1_1}
        \begin{aligned}
            I(S, T; \mathcal{D}, \mathcal{D}_n) &= I(S, T; \mathcal{D}) + I(S, T; \mathcal{D}_n | \mathcal{D}) \\
            &=I(S, T; \mathcal{D}).
        \end{aligned}   
    \end{equation}

    By an alternative expansion, we obtain:
    \begin{equation}
    \label{step1_2}
        I(S, T; \mathcal{D}, \mathcal{D}_n) = I(S, T; \mathcal{D}_n) + I(S, T; \mathcal{D} | \mathcal{D}_n).
    \end{equation}
    
    Combining Eq. \eqref{step1_1} and \eqref{step1_2}, we obtain the following equality:
    \begin{equation}
    \label{step1_3}
        I(S, T; \mathcal{D}) = I(S, T; \mathcal{D}_n) + I(S, T; \mathcal{D} | \mathcal{D}_n).
    \end{equation}

    \textbf{Step 2:}
    By expanding $I(S, T; Y,\mathcal{D}|\mathcal{D}_n)$, we obtain:
    \begin{equation}
        I(S, T; Y,\mathcal{D}|\mathcal{D}_n) = I(S, T; \mathcal{D}|\mathcal{D}_n) + I(S, T; Y | \mathcal{D},\mathcal{D}_n).
    \end{equation}
    Since we have $I(S, T; Y | \mathcal{D},\mathcal{D}_n)=0$:
    \begin{equation}
    \label{step2_1}
        I(S, T; Y,\mathcal{D}|\mathcal{D}_n) = I(S, T; \mathcal{D}|\mathcal{D}_n).
    \end{equation}

    Similarly, by expanding in another way:
    \begin{equation}
    \label{step2_2}
        I(S, T; Y,\mathcal{D}|\mathcal{D}_n) = I(S, T; Y|\mathcal{D}_n) + I(S, T; \mathcal{D} | Y,\mathcal{D}_n).
    \end{equation}

    Combining Eq. \eqref{step2_1} and \eqref{step2_2}, we obtain
    \begin{equation}
    \label{step2_3}
        \begin{gathered}
            I(S, T; \mathcal{D}|\mathcal{D}_n) = I(S, T; Y|\mathcal{D}_n) + I(S, T; \mathcal{D} | Y,\mathcal{D}_n) \\
            \Rightarrow I(S, T; \mathcal{D}|\mathcal{D}_n) \geq I(S, T; Y|\mathcal{D}_n).
        \end{gathered}
    \end{equation}
    
    \textbf{Step 3:}
    Substituting Eq. \eqref{step2_3} into Eq. \eqref{step1_3}, we obtain:
    \begin{equation}
        \begin{aligned}
            I(S, T; \mathcal{D}) &\geq I(S, T; \mathcal{D}_n) + I(S, T; Y|\mathcal{D}_n) \\
            & = I(S, T; \mathcal{D}_n) + I(S, T; Y).
        \end{aligned}
    \end{equation}
    
    Therefore, we obtain the conclusion:
    \begin{equation}
        I(\mathcal{D}_n; S, T) \leq -I(Y; S, T) + I(\mathcal{D}; S, T).
    \end{equation}
\end{proof}

\begin{lemma}[Feature Convergence]
    Assuming that $Y$ can potentially contain all information about $Y_r$ and $Y_n$, the following inequality holds when $\max(I(Y_n;S), I(Y_r;T))\leq \max(I(S;T), \varepsilon) / 2 = K, \varepsilon\in \mathbb{R}$ is satisfied: 
    \begin{equation}
        -I(Y_r;S)-I(Y_n;T)-\varepsilon\leq -I(Y;S,T)+I(S;T| Y).
    \end{equation}
\end{lemma}

\begin{proof}
    By expanding the mutual information and combining the assumptions, we have:
    \begin{equation}
        \begin{aligned}
            &I(Y; S, T) \\
            =& I(Y; S) + I(Y; T) - I(Y; S; T) \\
            =& I(Y; S) + I(Y; T) - \left(I(S; T) - I(S; T | Y)\right) \\
            =& I(Y; S) + I(Y; T) - I(S; T) + I(S; T | Y) \\
            =& I(Y_r, Y_n; S) + I(Y_r, Y_n; T) - I(S; T) + I(S; T | Y) \\
            \leq& I(Y_r; S) + I(Y_n; S) + I(Y_r; T) + I(Y_n; T) \\
            &- I(S; T) + I(S; T | Y).
        \end{aligned}
    \end{equation}
    
    Rearranging the terms, we obtain:
    \begin{equation}
        \begin{aligned}
        &-I(Y; S, T) + I(S; T | Y) \\
        \geq &-I(Y_r; S) - I(Y_n; T) + \left(I(S; T) - I(Y_n; S) - I(Y_r; T)\right).
        \end{aligned}
    \end{equation}
    
    Now, we consider two cases:

    \textbf{Case 1:} When $I(S, T) \geq \varepsilon$:
    \begin{equation}
        \begin{aligned}
            &-I(Y; S, T) + I(S; T | Y) \\
            \geq& -I(Y_r; S) - I(Y_n; T) + \left(I(S; T) - I(Y_n; S) - I(Y_r; T)\right) \\
            \geq& -I(Y_r; S) - I(Y_n; T).
        \end{aligned}
    \end{equation}

    \textbf{Case 2:} When $I(S, T) < \varepsilon$:
    \begin{equation}
        \begin{aligned}
            &-I(Y; S, T) + I(S; T | Y) \\
            \geq& -I(Y_r; S) - I(Y_n; T) + \left(I(S; T) - I(Y_n; S) - I(Y_r; T)\right) \\
            \geq& -I(Y_r; S) - I(Y_n; T) - I(Y_n; S) - I(Y_r; T) \\
            \geq& -I(Y_r; S) - I(Y_n; T) - \varepsilon.
        \end{aligned}
    \end{equation}
    Thus, the conclusion is proven.
\end{proof}

\subsection{The proof of Proposition 4.1 $\sim$ 4.4}
\label{appendix_prop_proof}
\begin{proposition}[The upper bound of $-I(Y;S,T)$]
    Given the label $Y$ and the variable $S,T$ that learns the characteristics of the real label space and the noise label space respectively, we have:
    \begin{equation}
        -I(Y;S,T)\leq -\max\left(I(Y;S), I(Y;T)\right).
    \end{equation}
\end{proposition}

\begin{proof}
Expand directly using the definition of mutual information:

    \begin{equation}
        \begin{aligned}
        I(Y;S,T) &= I(Y;S) + I(Y;T|S) \\
        &= I(Y;T) + I(Y;S|T) \\
        &\geq \max(I(Y;S), I(Y;T))
        \end{aligned}
    \end{equation}

So that $-I(Y;S,T)\leq -\max\left(I(Y;S), I(Y;T)\right)$.
\end{proof}

\begin{proposition}[The upper bound of $I(\mathcal{D};S,T)$]
    Let $\mathcal{D}$, $S$, $T$ be random variables. Assume the probabilistic mapping $p(\mathcal{D}, S, T)$ follows the Markov chain $S \leftrightarrow \mathcal{D} \leftrightarrow T$. Then:
    \begin{equation}
        I(\mathcal{D}; S, T) \leq I(\mathcal{D}; S) + I(\mathcal{D}; T).
    \end{equation}
\end{proposition}

\begin{proof}
Without loss of generality, here we take $\mathcal D=X$ as an example to prove it.

Expanded by definition and processed over the probability distributions, we can obtain:

\begin{equation}
    \begin{aligned}
    &I(\mathcal{D};S,T) \\
    =& \mathbb{E}_{p(x,s,t)} \log \frac{p(s,t,x)}{p(x)p(s,t)} \\
    =& \mathbb{E}_{p(x,s,t)} \log \frac{p(x)p(s|x)p(t|x)}{p(x)p(s,t)} \\
    =& \mathbb{E}_{p(x,s,t)} \log \left[ \frac{p(s|x)p(x)}{p(x)p(s)} \cdot \frac{p(t|x)p(x)}{p(x)p(t)} \cdot \frac{p(s)p(t)}{p(s,t)} \right] \\
    =& I(\mathcal{D};S) + I(\mathcal{D};T) - I(S;T) \\
    \leq &I(\mathcal{D};S) + I(\mathcal{D};T)
    \end{aligned}
\end{equation}

\end{proof}

\begin{proposition}[Reformulation of $I(S,T| Y)$]
    Given the label $Y$ and the variable $S,T$, minimizing $I(S;T | Y)$ is equivalent to minimize $I(S,Y;T,Y)$.
\end{proposition}

\begin{proof}
\begin{equation}
    \begin{aligned}
        &I(S;T|Y) \\
        =& \int_y p(y) \iint_{s,t} p(s,t|y) \log \frac{p(s,t|y)}{p(s|y)p(t|y)} \, ds \, dt \, dy \\
        =& \iiint_{(s,t,y)} p(s,t,y) \log \frac{p(s,t,y)}{p(s,y)p(t,y)} \, p(y) \, ds \, dt \, dy \\
        =& \mathbb{E}_{p(s,t,y)} \log \frac{p(s,t,y)}{p(s,y)p(t,y)} + \mathbb{E}_{p(y)} \log p(y) \\
        =& I(S,Y;T,Y) - H(Y)
    \end{aligned}
\end{equation}

Since \( H(Y) \) is a constant, it follows that \( I(S;T | Y) \propto I(S,Y;T,Y) \), thus proved.

\end{proof}

\begin{proposition}[Reformulation of the condition in Eq. (8): $\max(I(Y_n;S),I(Y_c;T))\leq K$]
    Minimizing $I(Y_r; T)$ and $I(Y_n; S)$ is equivalent to increase $I(Y_n; T)$ and $I(Y_r; S)$.
\end{proposition}

\begin{proof}
For $S$, with limited capacity $(I(Y;S)\le A)$, we have:
    \begin{equation}
        \begin{aligned}
        &I(Y_r;S)+I(Y_n;S)
        \le I(Y;S) +I(Y_r;Y_n)\le A+const
        \end{aligned}
    \end{equation}

Similarly, for $T$, enlarging one part constrains the other under limited capacity.
\end{proof}

\subsection{The proof of Eq. (10): $\max I(Y; Z)$ is equivalent to $\min \mathcal{L}_{CE}$}

\begin{proof}

    \begin{equation}
        \begin{aligned}
        I(Y;Z) &= \iint_{(y,z)} p(y,z) \log \frac{p(y,z)}{p(y)p(z)} \, dy \, dz \\
        &= \iint_{(y,z)} p(y,z) \log \frac{p(y|z)}{p(y)} \, dy \, dz \\
        &= \mathbb{E}_{p(y,z)}\left(\log(q_\theta(y|z))\right) + \\
        &\quad \ \mathbb{E}_{p(z)}(D_{KL}(p(y|z)\|q_\theta(y|z))) + H(Y) \\
        &\geq \mathbb{E}_{p(y,z)}\left(\log(q_\theta(y|z))\right) \:= -\mathcal{L}_{CE} (Z,Y),
        \end{aligned}
    \end{equation}
    where $q_{\theta_i}(\cdot)$ is variational approximation of $p(\cdot)$.
\end{proof}

\section{Implement Details}
\label{appendix_implement}
\subsection{Implement Details of $\mathcal L_{Minimal}$}
\label{appendix_IXZ}
As established in Proposition 4.2, minimizing the $I(\mathcal{D};S,T)$ objective reduces to minimizing $I(\mathcal{D}; Z)$, where $Z\in \{S,T\}$. We now discuss the methodology for estimating and optimizing $I(\mathcal{D}; Z)$.

\subsubsection{The input is an image.}
In this case, $\mathcal{D} = X$, i.e., the input to the model consists solely of image features. Therefore, minimizing $I(\mathcal{D}; Z)$ reduces to minimizing the divergence between the approximate posterior and a fixed prior.

\begin{proof}
    Expanded via the definition of mutual information:
    \begin{equation}
        \begin{aligned}
        &I(X;Z) \\
        =& \iint_{(x,z)} p(x,z) \log \frac{p(x,z)}{p(x)p(z)} \, dx \, dz \\
        =& \int_x p(x) \left[ \int_z p(z|x) \log \frac{p(z|x)}{q(z)} \, dz \right] dx - \int_z p(z) \log \frac{p(z)}{q(z)} \, dz \\
        =& \int_x p(x) \left[ D_{KL}[p(z|x) \parallel q(z)] \right] dx - D_{KL}[p(z) \parallel q(z)] \\
        \leq & \mathbb{E}_{p(x)} \left[ D_{KL}[p(z|x) \parallel q(z)] \right].
        \end{aligned}
    \end{equation}   
\end{proof}

Since $q(z)$ represents the marginal distribution of the latent variable and is not constrained during training, we can assume without loss of generality that $q(z) \sim \mathcal{N}(0, I_n)$. The encoder maps input features to a Gaussian distribution, i.e., $p(z|x) \sim \mathcal{N}(\mu_n, \sigma_n)$. In this case, the KL divergence between two Gaussian distributions admits a closed-form solution and can be computed as:
\begin{equation}
D_{\mathrm{KL}} [p(z|x) | q(z)] = \frac{1}{2} \sum_{i=1}^{n} \left( \mu_i^2 + \sigma_i^2 - \log \sigma_i^2 - 1 \right).
\end{equation}

\subsubsection{The input is a graph.}
In this case, the input data is the graph $\mathcal{D} = \mathcal{G} = (X, A)$, where $X$ denotes node features and $A$ denotes the adjacency matrix. As a result, the estimation of the mutual information $I(\mathcal{D}; Z)$ becomes more intricate compared to scenarios with only feature inputs. 

Following the framework of Graph Information Bottleneck (GIB), we consider two groups of indices $S_X, S_A \subseteq [L]$ that satisfy the Markovian dependence. Specifically, we assume $\mathcal{D} \perp Z_X^{(L)} \setminus \{Z_X^{(l)}\}_{l \in S_X} \cup \{Z_A^{(l)}\}_{l \in S_A}$, where $Z_X^{(l)}$ denotes the node feature representation at layer $l$, and $Z_A^{(l)}$ denotes the structural representation at layer $l$. Based on this condition, for any set of distributions $\mathbb{Q}(Z_X^{(l)})$ with $l \in S_X$, and $\mathbb{Q}(Z_A^{(l)})$ with $l \in S_A$, the following upper bound holds:
\begin{equation}
I(\mathcal{D}; Z_X^{(L)}) \leq \sum_{l \in S_A} \text{AIB}^{(l)} + \sum_{l \in S_X} \text{XIB}^{(l)},
\end{equation}
where $\text{AIB}^{(l)}$ and $\text{XIB}^{(l)}$ denote the information contributions from the adjacency and feature paths, respectively, and are defined as:
\begin{equation}
\begin{aligned}
\text{AIB}^{(l)} &= \mathbb{E} \left[ \log \frac{\mathbb{P}(Z_A^{(l)} | A, Z_X^{(l-1)})}{\mathbb{Q}(Z_A^{(l)})} \right], \\
\text{XIB}^{(l)} &= \mathbb{E} \left[ \log \frac{\mathbb{P}(Z_X^{(l)} | Z_X^{(l-1)}, Z_A^{(l)})}{\mathbb{Q}(Z_X^{(l)})} \right].
\end{aligned}
\end{equation}

For the structural branch, we adopt a Bernoulli-based KL divergence estimator:
\begin{equation}
\widehat{\text{AIB}}^{(l)} = \sum_{v \in V, t \in [\mathcal T]} D_{KL} \left( \text{Bernoulli}(\phi^{(l)}_{v,t}) || \text{Bernoulli}(\alpha) \right),
\end{equation}
where $\phi^{(l)}{v,t}$ denotes the probability of an edge between node $v$ and its $t$-hop neighbors, and $\alpha$ is the prior class probability.

To estimate $\text{XIB}^{(l)}$, we model $\mathbb{Q}(Z_X^{(l)})$ as a mixture of Gaussians with learnable parameters. For any node $v$, we assume $Z_{X,v}^{(l)} \sim \sum_{i=1}^m w_i \mathcal{N}(\mu_{0,i}, \sigma_{0,i}^2)$, where $w_i$, $\mu_{0,i}$ and $\sigma_{0,i}$ are shared parameters across all nodes, and $Z_{X,v} \perp Z_{X,u}$ if $v \neq u$. We compute:
\begin{equation}
\begin{aligned}
\widehat{\text{XIB}}^{(l)} =& \log \frac{\mathbb{P}(Z_X^{(l)} \mid Z_X^{(l-1)}, Z_A^{(l)})}{\mathbb{Q}(Z_X^{(l)})} \\
=&\sum_{v \in V} \log \Phi(Z^{(l)}_{X,v}; \mu_v, \sigma_v^2) \\
&- \log \left( \sum_{i=1}^m w_i \Phi(Z^{(l)}_{X,v}; \mu_{i}, \sigma_{0,i}^2) \right),
\end{aligned}
\end{equation}
where $\Phi(\cdot)$ denotes the probability density function of a Gaussian distribution.

In conclusion, we select proper sets of indices $S_X,S_A$ and use substitution:
\begin{equation}
    I(\mathcal D, Z) \leftarrow \sum_{l \in S_A} \widehat{\text{AIB}}^{(l)} + \sum_{l \in S_X} \widehat{\text{XIB}}^{(l)}.
\end{equation}

More detailed content and proof can be found in GIB~\cite{GIB}.

\subsection{Implement Details of Knowledge Injection}
Building on the proven success of mixup-based techniques in computer vision, we apply FMix~\cite{fmix} for image data augmentation across both clean and noisy subsets. To counteract potential class bias during training, we incorporate Debiased Margin-based Loss~\cite{promix} and Debiased Pseudo Labeling~\cite{menonlong}, fostering unbiased model predictions.

\subsection{Implement Details of Robust Training}
To further improve model performance, we perform certain modifications to the label $Y$ during robust training period.

For image classification, every $k$ epochs, the model predictions are combined with the original labels using exponential moving average to replace the original labels.

For node classification, at each epoch, the labels of samples with prediction confidence higher than a threshold $\tau$ are replaced with the model’s predicted results.

\section{Experiments Details and Results}
\label{appendix_exp}
\subsection{Data}
\label{appendix_data}
\subsubsection{Image Classification.}
We select CIFAR-based datasets to simulate both synthetic and real-world label noise. To mitigate the risk of overfitting to a specific dataset, we also include the Animal-10N dataset. Specifically:

\begin{itemize}
    \item \textbf{CIFAR-10/100~\cite{cifar}\footnote{https://www.cs.toronto.edu/\textasciitilde kriz/cifar.html}:} These are classic image classification datasets with 10 and 100 categories, respectively. By constructing a noise transition matrix, we introduce symmetric noise (Sym Noise, where each noisy label is uniformly sampled from all classes) and asymmetric noise (Asym Noise, where each class is flipped to a specific incorrect class based on semantic similarity). In our experiments, we consider symmetric noise with noise rates of 20\% and 50\%, which are \textbf{independent of the input features}. Additionally, based on class-wise correlations, we design a 40\% asymmetric noise setting on CIFAR-10.
    
    \item \textbf{CIFAR-10N/100N\cite{CIFARN}\footnote{https://github.com/UCSC-REAL/cifar-10-100n/tree/main}:} These datasets introduce real-world label noise based on standard CIFAR-10 and CIFAR-100. Each image is annotated by multiple human workers, and the final label is obtained via majority voting, thereby reflecting more natural and realistic label noise. Previous studies have shown that \textbf{the noise is not independent of the input features}. CIFAR-10N includes five noise levels (Aggregate: 9.03\%, Random 1: 17.23\%, Random 2: 18.12\%, Random 3: 17.64\%, Worst: 40.21\%), while CIFAR-100N includes one noise level (40.20\%).
    
    \item \textbf{Animal-10N\cite{animal-10N}\footnote{https://dm.kaist.ac.kr/datasets/animal-10n/}:} This dataset is constructed based on animal categories from ImageNet. It consists of images from 10 common animal classes collected and labeled by non-expert annotators. The label noise mainly stems from \textbf{confusion between fine-grained categories}, such as dog vs. wolf or cow vs. horse. The estimated noise rate is around 8\%.
\end{itemize}

A more intuitive comparison of the selected datasets is presented in Table~\ref{dataset}.

\begin{table}[H]
    \centering
    \hspace{-0.17cm}
    \resizebox{0.48\textwidth}{!}{
    \begin{tabular}{lccc}
        \toprule
        \textbf{Dataset} & \textbf{\# Class} & \textbf{Noise Type} & \textbf{Noise Ratio} \\
        \midrule
        CIFAR & 10 / 100 & Sym/Asym & 20\%, 40\%, 50\% \\
        CIFARN & 10 / 100 & Real-world & 9.03\% $\sim$ 40.21\% \\
        Animal-10N & 10 & Real-world & $\approx$8\% \\
        \bottomrule
    \end{tabular}
    }
    \caption{Comparison of image datasets}
    \label{dataset}
\end{table}

\subsubsection{Node Classification.}
We select three classic citation network datasets: Cora, Citeseer, and Pubmed. In addition, we include the one author collaboration network DBLP for evaluation. To ensure consistency across methods, we randomly sample 20 nodes per class for training. For validation and testing, 500 and 1000 nodes are randomly selected from the graph, respectively. Specifically:

\begin{itemize}
    \item \textbf{Cora, Citeseer and Pubmed~\cite{G_data}\footnote{https://github.com/kimiyoung/planetoid/tree/master/data}:} These citation network datasets are widely adopted in graph learning research involving label noise. In each dataset, nodes correspond to academic papers, and edges represent citation connections among them. The node features consist of binary word vectors indicating whether particular words from a vocabulary are present or absent. Each node is labeled according to the research topic category of the corresponding paper.
    
    \item \textbf{DBLP~\cite{dblp}\footnote{https://github.com/abojchevski/graph2gauss/raw/master/data}:} This dataset represents an author collaboration network within the field of computer science. Nodes represent documents, while edges correspond to citation relationships between these documents. Node features are derived from word vectors extracted from the text, and labels reflect the category of the research topic.
\end{itemize}

A more intuitive comparison of the selected datasets is presented in Table~\ref{graph_dataset}.

\begin{table}[H]
    \centering
    \begin{tabular}{lcccc}
        \toprule
        \textbf{Dataset} & \textbf{\# Class} & \textbf{\# Node} & \textbf{\# Edge} & \textbf{\# Feat.} \\
        \midrule
        Cora & 7 & 2,708 & 5,278 & 1,433 \\
        Citeseer & 6 & 3,327 & 4,552 & 3,703 \\
        Pubmed & 3 & 19,717 & 44,324 & 500 \\
        DBLP & 4 & 17,716 & 52,867 & 1,639 \\
        \bottomrule
    \end{tabular}
    \caption{Comparison of graph datasets}
    \label{graph_dataset}
\end{table}

\subsection{Baselines}
We compare with four categories, 16 baselines in two scenarios: \ding{172} Classic IB methods; \ding{173} IB with robust loss functions; \ding{174} Improved IB variants; \ding{175} Two-stage denoising + IB methods. They comprehensively evaluate our \methods performance from multiple perspectives. Specifically:

\subsubsection{Image Classification.}
\begin{itemize}
    \item Classical IB Models: This paper selects two IB methods, \textbf{VIB}~\cite{VIB} and \textbf{NIB}~\cite{NIB}, as baseline approaches. The core idea of VIB is to approximate the optimization of the IB objective using variational inference. NIB further extends the original IB principle to address the limitations of KL divergence estimation in VIB, which is restricted by the assumptions and simplicity of the prior distribution. Instead of relying on an analytical KL computation, NIB adopts kernel density estimation (KDE) for learning.
    \item IB with robust loss functions: In the IB framework, the mutual information $I(Z;Y)$ is typically optimized via cross-entropy. The \textbf{Generalized Cross-Entropy (GCE)} loss~\cite{GCE} combines cross-entropy and mean absolute error to enhance robustness. The \textbf{Symmetric Cross-Entropy (SCE)} loss~\cite{SCE} integrates standard cross-entropy with reverse cross-entropy, improving resistance to label noise. By replacing the original loss with these robust alternatives, new IB variants are constructed under robust loss functions.
    \item Improved IB Methods: Many subsequent studies have enhanced the feature extraction capability and robustness of IB. The \textbf{SIB}~\cite{SIB} framework leverages an auxiliary encoder to capture missing structural information, improving learning performance. \textbf{DT-JSCC}~\cite{DT-JSCC} proposes a robust encoding architecture based on the IB principle to improve transmission resilience under varying communication channels. These two methods represent improved IB frameworks and are included in the baseline comparisons.
    \item Two-Stage Denoise + IB Methods: Although Theorem 3.1 indicates that denoising followed by IB may not yield optimal results, this study further explores this approach. Three representative denoising methods are adopted: \textbf{JoCoR}~\cite{JoCoR}, a robust learning method based on co-training; \textbf{ELR+}~\cite{ELR}, which uses noise-robust regularization for label correction; and \textbf{ProMix}~\cite{promix}, which integrates Mixup data augmentation and dynamic confidence modeling, representing a state-of-the-art denoising approach. These denoised datasets are subsequently used for IB training.
\end{itemize}

\subsubsection{Node Classification.}
\begin{itemize}
    \item Classical IB Models: This paper selects the \textbf{GIB}~\cite{GIB} method as the classical IB baseline. GIB leverages the IB principle by learning graph representations that compress input feature and structure while preserving label-relevant information. 
    \item IB with robust loss functions: Consistent with the robust loss settings used in the image classification task.
    \item Improved IB Methods: Two representative improvements are included. \textbf{CurvGIB}~\cite{CurvGIB} introduces discrete Ricci curvature into the IB framework to better capture topological structures in graphs, enabling the model to discard spurious connectivity information and preserve label-relevant substructures. \textbf{IS-GIB}~\cite{IS-GIB} enhances generalization under distribution shifts by jointly modeling individual and structural information bottlenecks, improving the robustness and transferability of graph representations.
    \item Two-Stage Denoise + IB Methods: For denoise model, \textbf{RNCGLN}~\cite{RNCGLN} first applies graph contrastive learning and multi-head self-attention to learn local-global representations, followed by pseudo-labeling strategies to address graph and label noise. \textbf{CGNN}~\cite{cgnn} performs neighborhood-based label correction and contrastive learning to smooth representations across graph views. It iteratively corrects noisy labels using neighbors’ predictions before applying the IB objective for final node classification.
\end{itemize}

\subsection{Preliminary Experiment}
\label{appendix_pre_exp}
\subsubsection{Vulnerability of IB Methods to Noisy Labels.}
To investigate the sensitivity of IB methods to label noise, we first conduct preliminary experiments on image classification and node classification to examine the relationship between IB performance and label corruption.

For image classification, Figure~\ref{fig:IC-a} demonstrates that the VIB~\cite{VIB} framework suffers performance degradation on the CIFAR-10N~\cite{CIFARN} dataset as label noise increases. Moreover, the decline becomes more pronounced with higher noise levels. In extreme cases, excessive noise can even lead to training collapse under the IB framework as shown in Figure~\ref{fig:IC-b}.

\begin{figure}[ht]
\centering
\begin{subfigure}[b]{0.49\linewidth}
    \includegraphics[height=3.5cm]{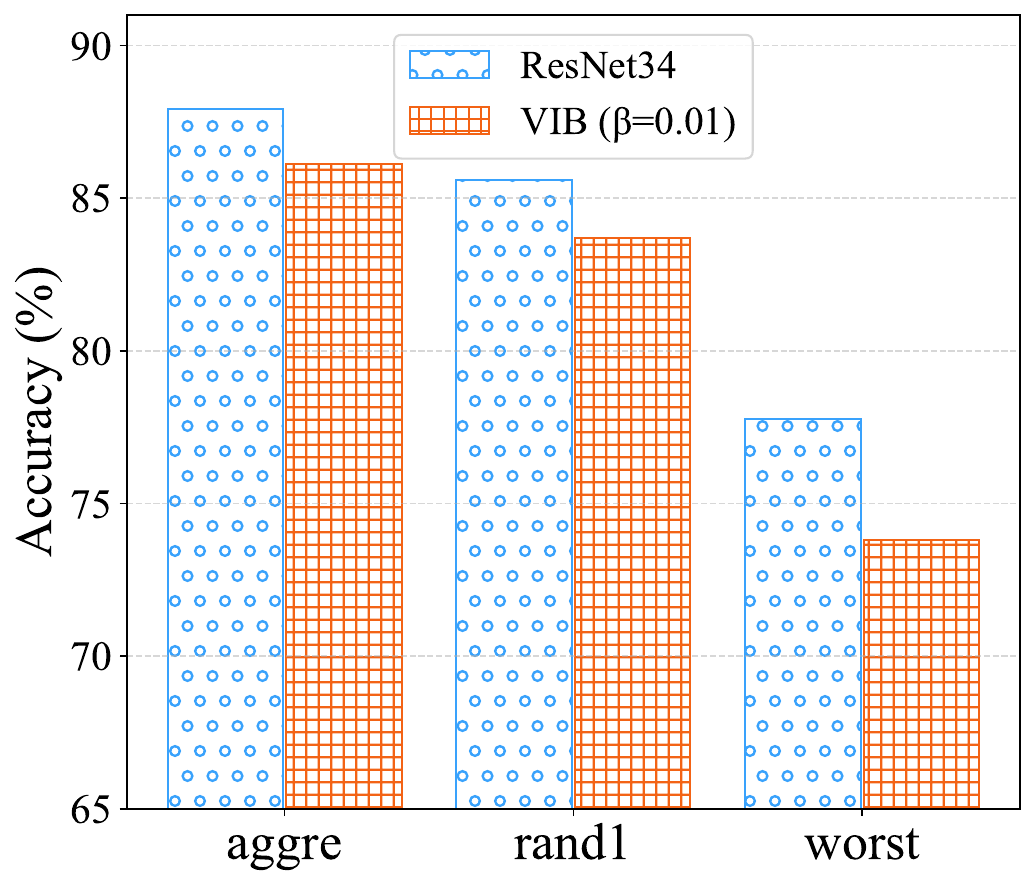}
    \caption{With feature-dependent noise}
    \label{fig:IC-a}
\end{subfigure}
\begin{subfigure}[b]{0.47\linewidth}
    \includegraphics[height=3.43cm]{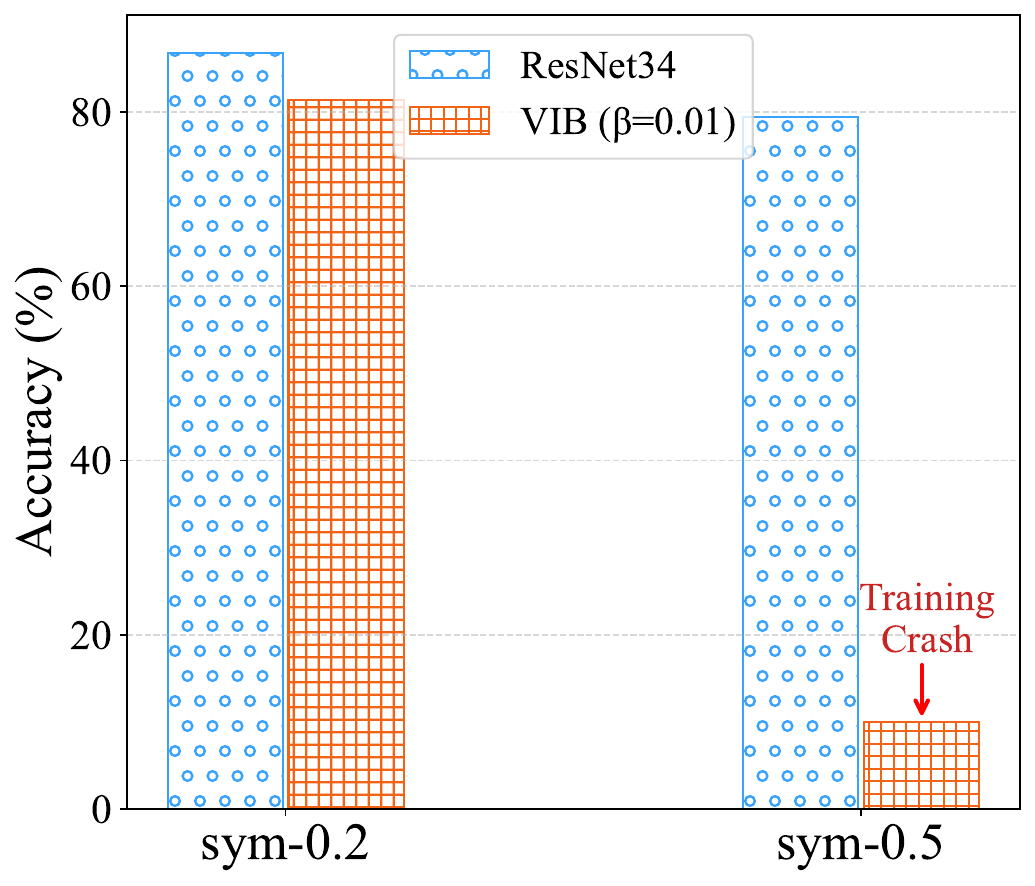}
    \caption{With symmetric noise}
    \label{fig:IC-b}
\end{subfigure}
\caption{Preliminary Experiments on Image Classification.}
\label{fig:toy-IC}
\end{figure}

For node classification, Figure~\ref{fig:NC-a} shows that GIB~\cite{GIB} gradually fits the noisy labels during training, leading to a steady decline in performance on the testing set. Furthermore, Figure~\ref{fig:NC-b}, using the Cora dataset as an example, illustrates that this trend becomes increasingly severe as the level of label noise increases.

\begin{figure}[ht]
\centering
\begin{subfigure}[b]{0.47\linewidth}   
    \includegraphics[height=3.5cm]{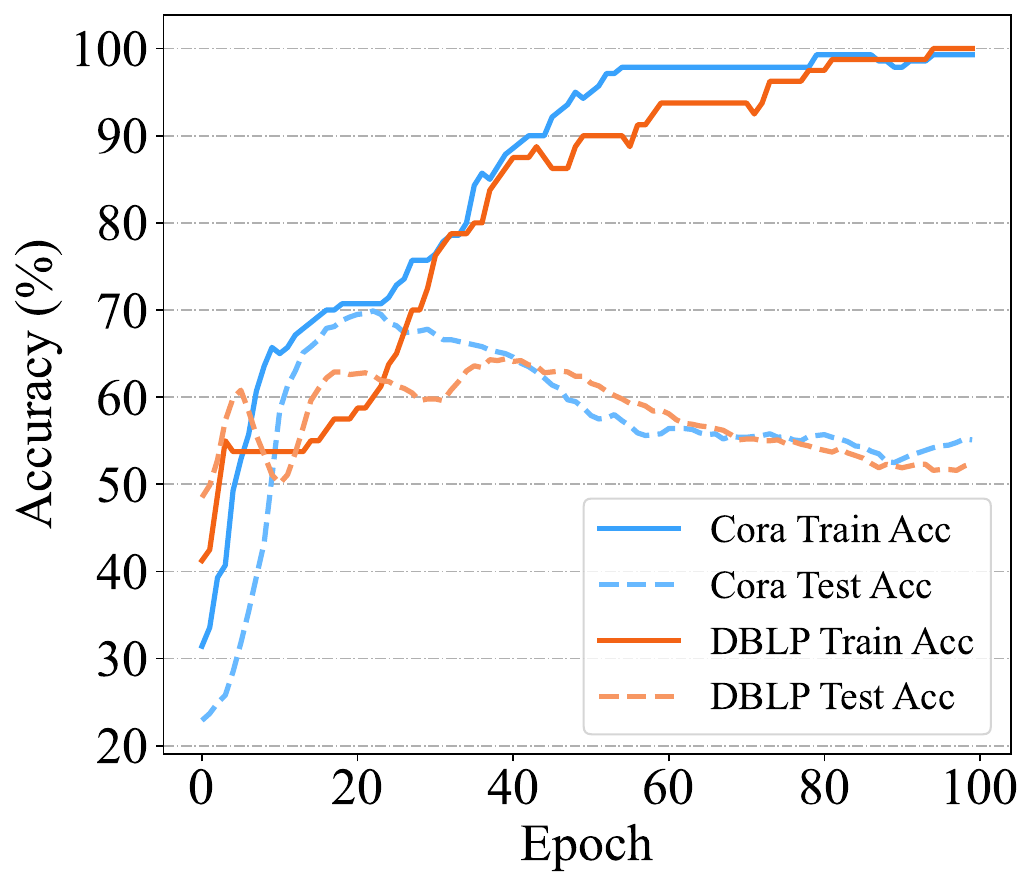}
    \caption{With 40\% uniform noise}
    \label{fig:NC-a}
\end{subfigure}
\begin{subfigure}[b]{0.47\linewidth}
    \includegraphics[height=3.5cm]{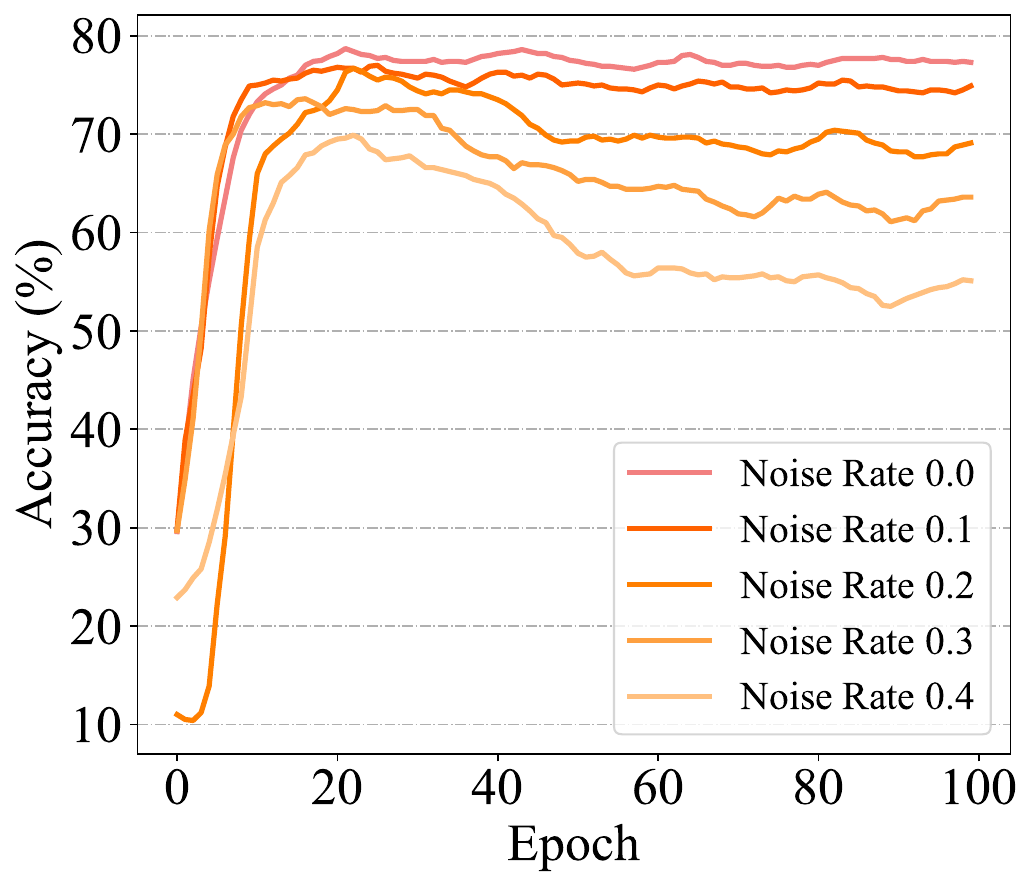}
    \caption{Cora (varying noise levels)}
    \label{fig:NC-b}
\end{subfigure}
\caption{Preliminary Experiments on Node Classification.}
\label{fig:toy-NC}
\end{figure}

\subsubsection{Performance Degradation in Cascaded Models.}
Theorem 3.1 shows that cascading models weakens the denoising effect of the first stage. This phenomenon is further illustrated in Figure~\ref{fig:decay}, where the data is first denoised using the ELR+ model and then used to train with the IB method. The resulting performance often falls short of the accuracy achieved after denoising alone.

\begin{figure}[H]
    \centering
    \includegraphics[height=4.3cm]{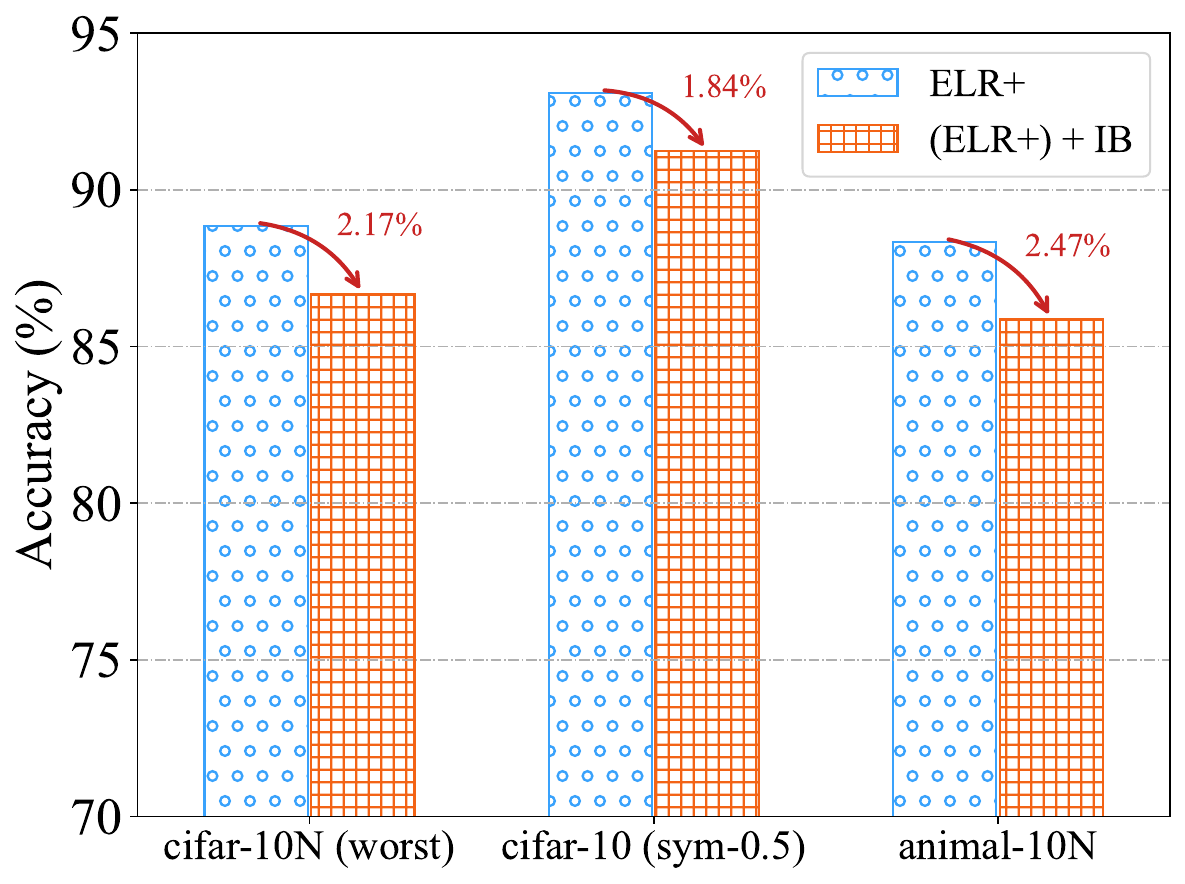}
\caption{Performance degradation.}
\label{fig:decay}
\end{figure}

\subsection{Additional Results}
\label{appendix_results}
\begin{table*}[!t]
  \centering
    \setlength{\tabcolsep}{4.5mm}
    \begin{tabular}{ccccccc}
    \toprule
    \multirow{2}{*}{\textbf{Method}} & \multirow{2}{*}{\textbf{Model}} 
      & \multicolumn{3}{c}{\textbf{CIFAR-10}} & \multicolumn{2}{c}{\textbf{CIFAR-100}} \\
    \cmidrule(r){3-5} \cmidrule(l){6-7}
      & & 20\%(sym) & 50\%(sym) & 40\%(asym) & 20\%(sym) & 50\%(sym) \\
    \midrule
    \multirow{2}{*}{\makecell{\textbf{Classic}\\\textbf{IB}}} & VIB & 81.49\textsubscript{$\pm$1.00} & 72.58\textsubscript{$\pm$1.63} & 79.27\textsubscript{$\pm$1.13} & 53.86\textsubscript{$\pm$0.47} & 44.25\textsubscript{$\pm$0.98} \\
    ~ & NIB &  83.44\textsubscript{$\pm$0.83} & 76.16\textsubscript{$\pm$1.27} & 78.16\textsubscript{$\pm$1.32} & 55.99\textsubscript{$\pm$0.79} & 46.20\textsubscript{$\pm$0.77} \\
    \midrule
    \multirow{2}{*}{\makecell{\textbf{Robust}\\\textbf{Loss}}} & VIB ($L_{GCE}$) & 88.43\textsubscript{$\pm$0.17} & 84.82\textsubscript{$\pm$0.33} & 81.32\textsubscript{$\pm$1.61} & ---        & --- \\
    ~ & VIB ($L_{SCE}$) & 84.36\textsubscript{$\pm$0.13} & 77.67\textsubscript{$\pm$1.40} & 76.59\textsubscript{$\pm$0.73} & 53.06\textsubscript{$\pm$1.83} & ---       \\
    \midrule
    \multirow{2}{*}{\makecell{\textbf{Improved}\\\textbf{IB}}} & SIB & 86.40\textsubscript{$\pm$0.55} & 65.52\textsubscript{$\pm$1.53} & 80.06\textsubscript{$\pm$1.52} & 57.64\textsubscript{$\pm$1.92} & 35.01\textsubscript{$\pm$1.23}\\
    ~ & DT-JSCC & 84.51\textsubscript{$\pm$0.41} & 72.49\textsubscript{$\pm$0.77} & 80.95\textsubscript{$\pm$0.41} & 49.58\textsubscript{$\pm$0.13} & 35.80\textsubscript{$\pm$0.96}\\
    \midrule
    \multirow{3}{*}{\makecell{\textbf{Deniose}\\\textbf{+ IB}}} & JoCoR+VIB &88.71\textsubscript{$\pm$0.18} & 81.71\textsubscript{$\pm$0.21} & 60.66\textsubscript{$\pm$0.08} & 62.61\textsubscript{$\pm$0.27} & 53.69\textsubscript{$\pm$0.11} \\
    ~ & (ELR+)+VIB & \underline{93.16\textsubscript{$\pm$0.23}} & \underline{91.28\textsubscript{$\pm$0.06}} & 84.75\textsubscript{$\pm$0.11} & 71.44\textsubscript{$\pm$0.93} & 54.12\textsubscript{$\pm$0.20} \\
    ~ & Promix+VIB & 92.98\textsubscript{$\pm$0.14} & \textbf{92.40\textsubscript{$\pm$0.10}} & \textbf{91.87\textsubscript{$\pm$0.05}} & 71.89\textsubscript{$\pm$0.16} & \textbf{69.77\textsubscript{$\pm$0.56}} \\
    \midrule
    \multirow{1}{*}{\textbf{Ours}} & \method & \textbf{94.56\textsubscript{$\pm$0.12}} & {91.13\textsubscript{$\pm$0.16}} & \underline{88.89\textsubscript{$\pm$0.73}} & \textbf{75.79\textsubscript{$\pm$0.19}} & \underline{67.28\textsubscript{$\pm$0.55}} \\
    \bottomrule
    \end{tabular}
  \caption{Classification accuracy (\%) of CIFAR under Symmetric/Asymmetric Noise. All the best results are highlighted in \textbf{bold}, and the second-best results are \underline{underlined}.}
  \label{tab:cifar}
\end{table*}

\begin{table*}[t]
  \centering
    \setlength{\tabcolsep}{1.2mm}
    {\small
  \begin{tabular}{cc|ccccc|cccc}
    \toprule
    \multirow{2}{*}{\textbf{Method}} & \multirow{2}{*}{\textbf{Model}} & \multirow{2}{*}{\textbf{Clean}}
      & \multicolumn{4}{c|}{\textbf{Uniform Noise}} 
      & \multicolumn{4}{c}{\textbf{Pair Noise}} \\
    \cmidrule(lr){4-7} \cmidrule(l){8-11}
      & & & 10\% & 20\% & 30\% & 40\% & 10\% & 20\% & 30\% & 40\% \\
    \midrule
    \multirow{1}{*}{\makecell{\textbf{Classic}\textbf{IB}}} 
      & GIB & \textbf{75.33\textsubscript{$\pm$3.19}} & \underline{75.10\textsubscript{$\pm$2.48}} & \textbf{73.03\textsubscript{$\pm$5.67}} & \underline{72.77\textsubscript{$\pm$2.36}} & 57.17\textsubscript{$\pm$7.83} & 74.73\textsubscript{$\pm$4.39} & 69.80\textsubscript{$\pm$6.18} & \underline{66.30\textsubscript{$\pm$6.60}} & \underline{46.60\textsubscript{$\pm$5.45}} \\
    \midrule
    \multirow{2}{*}{\makecell{\textbf{Robust}\\\textbf{Loss}}}
      & GIB ($\mathcal L_{GCE}$) & 75.03\textsubscript{$\pm$2.79} & 74.43\textsubscript{$\pm$3.01} & 72.03\textsubscript{$\pm$6.91} & 72.27\textsubscript{$\pm$4.07} & 57.57\textsubscript{$\pm$6.42} & 73.67\textsubscript{$\pm$3.85} & 71.03\textsubscript{$\pm$4.62} & 62.60\textsubscript{$\pm$6.99} & 43.63\textsubscript{$\pm$6.69} \\
      & GIB ($\mathcal L_{SCE}$) & 74.23\textsubscript{$\pm$3.56} & 72.67\textsubscript{$\pm$2.55} & 71.10\textsubscript{$\pm$6.45} & 70.63\textsubscript{$\pm$4.03} & \underline{58.47\textsubscript{$\pm$5.17}} & 73.33\textsubscript{$\pm$4.64} & 70.07\textsubscript{$\pm$4.93} & 59.87\textsubscript{$\pm$6.40} & 43.10\textsubscript{$\pm$6.80} \\
    \midrule
    \multirow{1}{*}{\makecell{\textbf{Improved}\\\textbf{IB}}} 
      & CurvGIB & 70.67\textsubscript{$\pm$3.23} & 67.67\textsubscript{$\pm$2.71} & 64.63\textsubscript{$\pm$6.52} & 61.97\textsubscript{$\pm$4.46} & 54.47\textsubscript{$\pm$6.13} & 66.93\textsubscript{$\pm$1.79} & 64.03\textsubscript{$\pm$4.55} & 56.27\textsubscript{$\pm$7.96} & 45.03\textsubscript{$\pm$1.68} \\
      & IS-GIB & 54.73\textsubscript{$\pm$0.15} & 53.17\textsubscript{$\pm$1.48} & 42.93\textsubscript{$\pm$2.46} & 45.53\textsubscript{$\pm$0.79} & 38.77\textsubscript{$\pm$6.34} & 49.73\textsubscript{$\pm$0.82} & 46.97\textsubscript{$\pm$1.19} & 40.33\textsubscript{$\pm$3.64} & 38.80\textsubscript{$\pm$1.22} \\
    \midrule
    \multirow{2}{*}{\makecell{\textbf{Denoise}\\\textbf{+ IB}}} 
      & RNCGLN+GIB & \underline{74.70\textsubscript{$\pm$2.65}} & 73.37\textsubscript{$\pm$0.65} & \underline{72.97\textsubscript{$\pm$5.09}} & 70.80\textsubscript{$\pm$2.57} & 52.27\textsubscript{$\pm$6.65} & \underline{75.00\textsubscript{$\pm$3.45}} & \underline{69.93\textsubscript{$\pm$3.40}} & 64.67\textsubscript{$\pm$6.10} & 39.83\textsubscript{$\pm$7.13} \\
      & CGNN+GIB & 73.17\textsubscript{$\pm$1.86} & 69.37\textsubscript{$\pm$6.01} & 67.67\textsubscript{$\pm$5.28} & 67.70\textsubscript{$\pm$1.85} & 54.70\textsubscript{$\pm$5.06} & 66.53\textsubscript{$\pm$5.74} & 64.80\textsubscript{$\pm$6.16} & 52.00\textsubscript{$\pm$5.40} & 44.33\textsubscript{$\pm$4.74} \\
    \midrule
    \textbf{Ours} & \method & 74.57\textsubscript{$\pm$0.99} & \textbf{75.87\textsubscript{$\pm$0.33}} & 71.13\textsubscript{$\pm$2.22} & \textbf{75.60\textsubscript{$\pm$0.71}} & \textbf{62.53\textsubscript{$\pm$7.56}} & \textbf{75.43\textsubscript{$\pm$1.04}} & \textbf{73.70\textsubscript{$\pm$3.10}} & \textbf{68.57\textsubscript{$\pm$3.02}} & \textbf{53.00\textsubscript{$\pm$7.26}} \\
    \bottomrule
  \end{tabular}
  }
  \caption{Classification accuracy (\%) on the DBLP dataset under different noise types and noise rates. All the best results are highlighted in \textbf{bold}, and the second-best results are \underline{underlined}.}
  \label{tab:dblp}
\end{table*}

\begin{table*}[!t]
  \centering
    \setlength{\tabcolsep}{1.9mm}
    {\small
  \begin{tabular}{cc|cc|cc|cc|cc}
    \toprule
    \multirow{3}{*}{\textbf{Method}} & \multirow{3}{*}{\textbf{Model}} & \multicolumn{4}{c|}{\textbf{Cora}} & \multicolumn{4}{c}{\textbf{Citeseer}}\\
    \cmidrule(lr){3-6} \cmidrule(l){7-10}
    & & \multicolumn{2}{c|}{\textbf{Uniform}} & \multicolumn{2}{c|}{\textbf{Pair}} & \multicolumn{2}{c|}{\textbf{Uniform}} & \multicolumn{2}{c}{\textbf{Pair}}\\
    \cmidrule(lr){3-6} \cmidrule(l){7-10}
      & &20\% & 40\% &20\% & 40\% &20\% & 40\% &20\% & 40\% \\
    \midrule
    \multirow{1}{*}{\makecell{\textbf{Classic}\textbf{IB}}} 
      & GIB & \textbf{75.37\textsubscript{$\pm$2.47}} & 
            71.60\textsubscript{$\pm$1.51} & 
            73.73\textsubscript{$\pm$0.39} & 
            65.27\textsubscript{$\pm$3.84} & 
            60.00\textsubscript{$\pm$3.37} & 
            48.20\textsubscript{$\pm$2.20} & 
            57.50\textsubscript{$\pm$2.35} & 
            \underline{43.67\textsubscript{$\pm$2.90}}\\
    \midrule
    \multirow{2}{*}{\makecell{\textbf{Robust}\\\textbf{Loss}}}
      & GIB ($\mathcal L_{GCE}$) & \underline{74.77\textsubscript{$\pm$0.76}} &
        70.47\textsubscript{$\pm$2.25} & 
        75.00\textsubscript{$\pm$1.77} & 
        65.27\textsubscript{$\pm$2.52} & 
        60.20\textsubscript{$\pm$3.93} & 
        50.80\textsubscript{$\pm$2.14} & 
        58.73\textsubscript{$\pm$2.31} & 
        41.03\textsubscript{$\pm$1.25} \\
      & GIB ($\mathcal L_{SCE}$) & 74.40\textsubscript{$\pm$0.45}& 
        69.17\textsubscript{$\pm$1.09} & 
        \underline{75.97\textsubscript{$\pm$2.02}} & 
        \underline{66.73\textsubscript{$\pm$4.08}} & 
        58.83\textsubscript{$\pm$5.19} & 
        50.93\textsubscript{$\pm$0.84} & 
        \textbf{59.10\textsubscript{$\pm$3.61}} & 
        41.40\textsubscript{$\pm$1.35} \\
    \midrule
    \multirow{1}{*}{\makecell{\textbf{Improved}\\\textbf{IB}}} 
      & CurvGIB & 65.90\textsubscript{$\pm$3.69} & 52.63\textsubscript{$\pm$3.23} & 67.17\textsubscript{$\pm$2.11} & 52.00\textsubscript{$\pm$3.19} & 49.80\textsubscript{$\pm$4.19} & 46.20\textsubscript{$\pm$1.69} & 48.77\textsubscript{$\pm$2.09} & 38.70\textsubscript{$\pm$3.70} \\
      & IS-GIB & 69.20\textsubscript{$\pm$0.62} & 54.73\textsubscript{$\pm$1.28} & 70.30\textsubscript{$\pm$0.99} & 56.47\textsubscript{$\pm$4.54} & 55.73\textsubscript{$\pm$3.44} & 39.03\textsubscript{$\pm$1.10} & 54.90\textsubscript{$\pm$4.28} & 40.13\textsubscript{$\pm$2.36}\\
    \midrule
    \multirow{2}{*}{\makecell{\textbf{Denoise}\\\textbf{+ IB}}} 
      & RNCGLN+GIB & 74.53\textsubscript{$\pm$1.58} & \underline{71.67\textsubscript{$\pm$1.49}} & 73.57\textsubscript{$\pm$1.59} & 63.60\textsubscript{$\pm$3.40} & \underline{60.90\textsubscript{$\pm$2.95}} & \underline{52.83\textsubscript{$\pm$4.82}} & 55.77\textsubscript{$\pm$3.17} & \textbf{46.00\textsubscript{$\pm$3.01}} \\
      & CGNN+GIB & 70.53\textsubscript{$\pm$4.69} & 64.73\textsubscript{$\pm$6.75} & 73.57\textsubscript{$\pm$1.37} & 59.00\textsubscript{$\pm$3.29} & 57.53\textsubscript{$\pm$3.73} & 45.73\textsubscript{$\pm$4.29} & 54.43\textsubscript{$\pm$3.30} & 41.57\textsubscript{$\pm$1.90} \\
    \midrule
    \textbf{Ours} & \method & 74.30\textsubscript{$\pm$2.01} & \textbf{74.07\textsubscript{$\pm$1.46}} & \textbf{76.87\textsubscript{$\pm$1.06}} & \textbf{66.80\textsubscript{$\pm$3.14}} & \textbf{61.63\textsubscript{$\pm$2.24}} & \textbf{55.17\textsubscript{$\pm$3.86}} & \underline{58.93\textsubscript{$\pm$2.77}} & \textbf{46.00\textsubscript{$\pm$0.71}}\\
    \bottomrule
  \end{tabular}
  }
  \caption{Classification accuracy (\%) on the Cora and Citeseer dataset under different noise types and noise rates. All the best results are highlighted in \textbf{bold}, and the second-best results are \underline{underlined}.}
  \label{tab:cora_citeseer}
\end{table*}

Table \ref{tab:cifar}, \ref{tab:dblp} and \ref{tab:cora_citeseer} presents a broader comparison of various \method methods under different types and levels of noise. We use a dash (--) to denote cases where the model fails or produces invalid results. Classification accuracy (\%) is used as the evaluation metric, where a higher value indicates better model performance.

For image classification, the model does not consistently outperform denoise + IB approaches in some cases. A possible explanation is that the denoising models are particularly effective on the CIFAR dataset, thereby significantly enhancing the IB performance. To further validate this, we also evaluated on the Animal-10N dataset, as shown in Table 2, where \method achieves the best performance. Moreover, Table 4 demonstrates that our method substantially outperforms two-stage approaches under adversarial attacks, further confirming the superiority of \method.

For the node classification task, it is evident that under high noise conditions, our model significantly outperforms other competitive baselines as shown in Table \ref{tab:dblp} and \ref{tab:cora_citeseer}, indicating the strong robustness of \method to label noise.

\subsection{Analysis of Sample Selection Strategy}
Our proposed sample selection strategy is based on two key components: mutual information and divergence. Mutual information is indirectly measured through the cross-entropy loss, whose effectiveness has been validated in numerous prior works~\cite{arpit2017a, song2019does}. As for divergence, following Observation 4.2, we adopt Jensen-Shannon (JS) divergence as a criterion to filter samples. This section primarily investigates how divergence affects the learning behavior of the encoder.

Figure~\ref{fig:sel_wo_div} shows the training process during the knowledge injection stage on the CIFAR-10N (worst) dataset, comparing the original method (sel\_all) with a variant that excludes divergence-based selection (sel\_w/o div). Figure~\ref{fig:sel_acc} presents the final performance of both methods. Experimental results demonstrate that removing divergence leads to a larger accuracy gap between the two encoders but ultimately worse performance. We attribute this to encoder $T$ failing to learn meaningful representations of noisy samples, rendering its predictions less informative. As a result, it is unable to provide effective guidance for feature separation in the third stage. These findings confirm that divergence-based sample selection plays a critical role in training the encoder effectively.

\begin{figure}[ht]
\centering
\begin{subfigure}[b]{0.472\linewidth}
    \includegraphics[height=3.5cm]{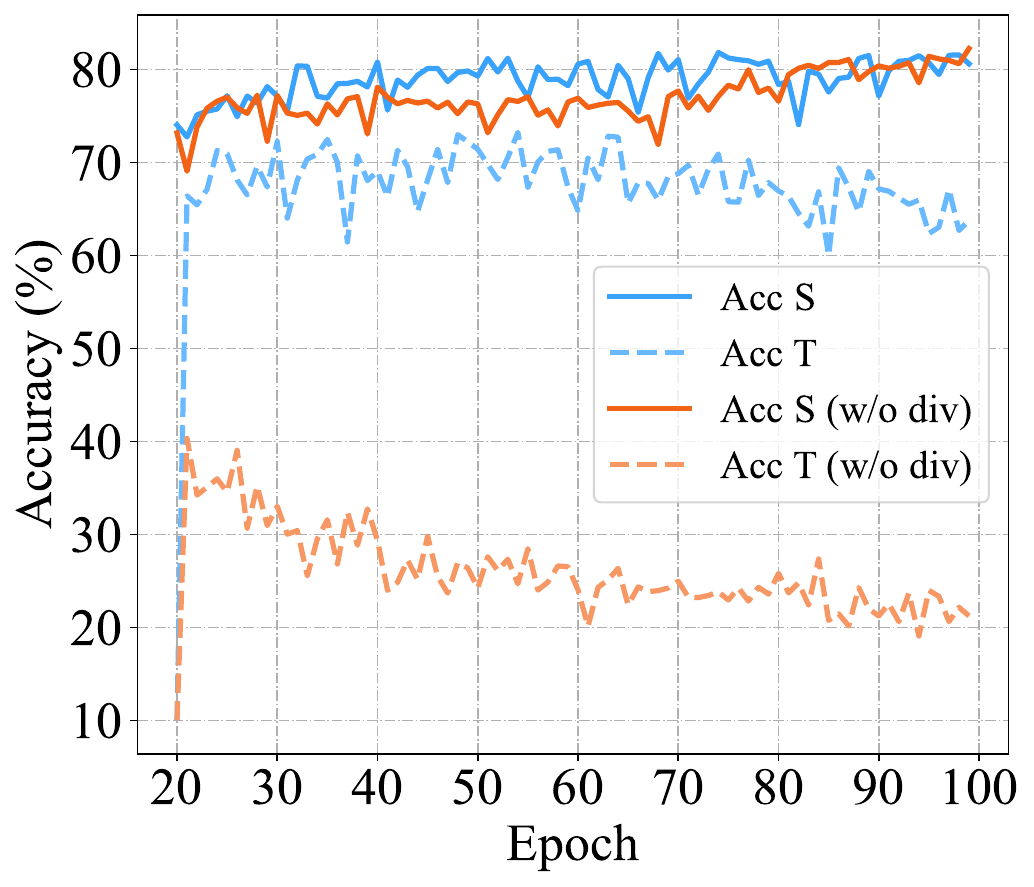}
    \caption{Training Process of Knowledge Injection}
    \label{fig:sel_wo_div}
\end{subfigure}
\begin{subfigure}[b]{0.48\linewidth}
    \includegraphics[height=3.5cm]{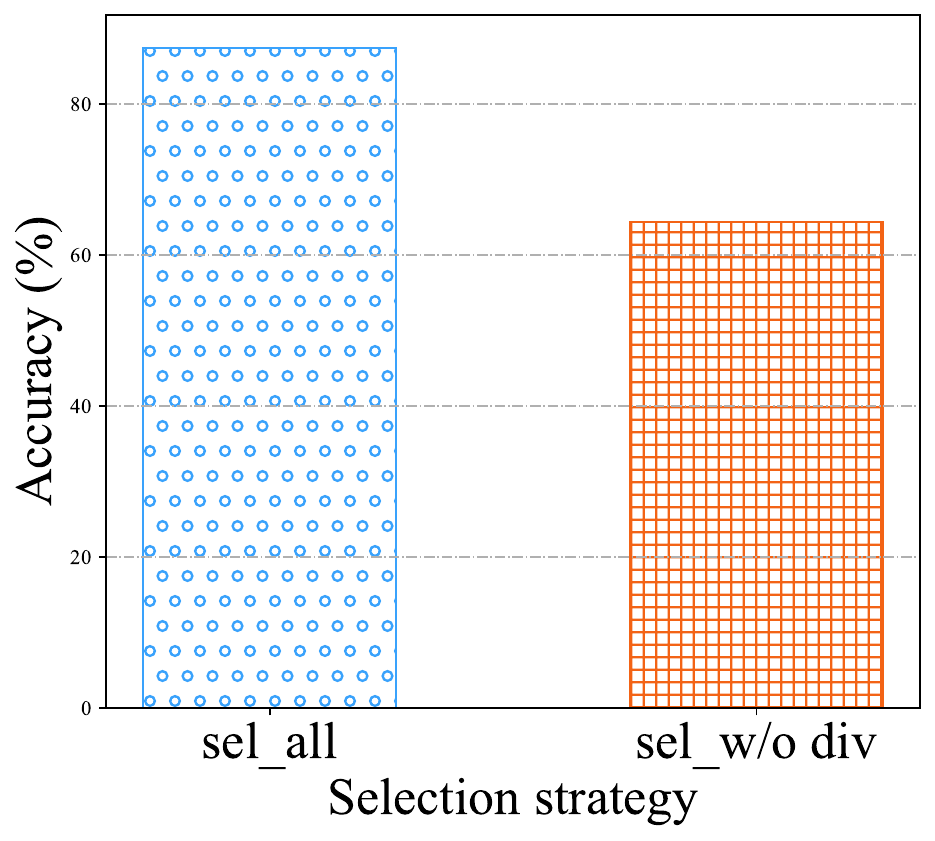}
    \caption{Training Accuracy under Different Select Strategy.}
    \label{fig:sel_acc}
\end{subfigure}
\caption{The influence of $D_{JS}$}
\label{fig:DJS}
\end{figure}

\subsection{Hyperparameter Sensitivity Analysis of $\delta$ in Algorithm~\ref{alg_sel}}
\label{hyper_delta}
In this section, we investigate the effect of the hyperparameter $\delta$, which controls the amount of knowledge the encoder $(S,T)$ acquires during the Knowledge Injection phase. Figure~\ref{fig:sel_ratio} illustrates the model's performance on the Cora and Citeseer datasets under various noise types and levels.

The results show that neither excessively high nor low values of $\delta$ yield optimal performance. We hypothesize that a too-small $\delta$ limits the encoder's exposure to training data, impairing its ability to learn useful representations. Conversely, a too-large $\delta$ causes the two encoders to converge in their learning, reducing their ability to distinguish between clean and noisy information effectively.

Furthermore, the influence of $\delta$ varies across datasets, indicating that the model's capacity to separate clean and noisy information is dataset-dependent.

\begin{figure}[ht]
\centering
\begin{subfigure}[b]{0.472\linewidth}
    \includegraphics[width=\linewidth]{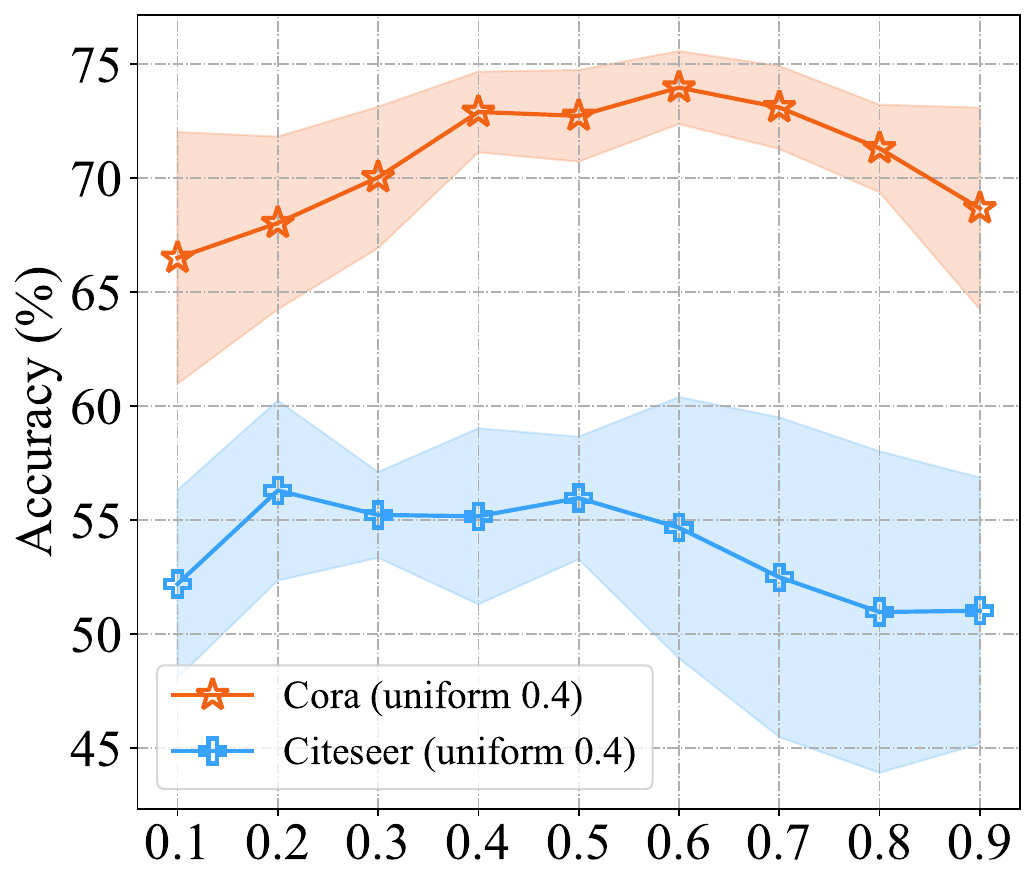}
    \caption{Under 40\% Uniform Noise}
    \label{fig:uniform0.4}
\end{subfigure}
\begin{subfigure}[b]{0.48\linewidth}
    \includegraphics[width=\linewidth]{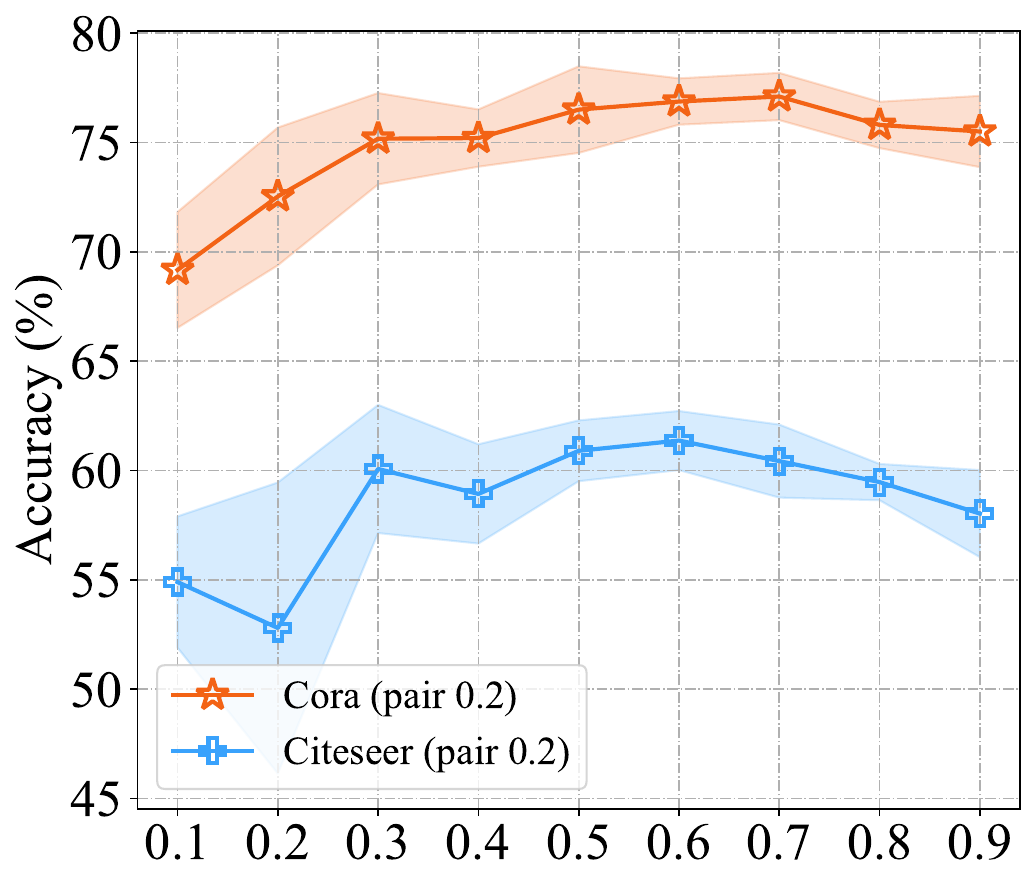}
    \caption{Under 20\% Pair Noise}
    \label{fig:pair0.2}
\end{subfigure}
\caption{The influence of $\delta$}
\label{fig:sel_ratio}
\end{figure}

\subsection{Analysis of Model Learning Behavior}
\label{appendix_analysis}
To further investigate the learning behavior of the encoder, we perform a t-SNE dimensionality reduction analysis on the embeddings obtained from models trained on the CIFAR-10N (worst) dataset. For a comprehensive comparison, we analyze the embeddings under the following four settings: \ding{172}~VIB without the $I(X;Z)$ constraint, which is approximately equivalent to a standard ResNet34 model; \ding{173}~Standard VIB model; \ding{174}~\method model at the end of Knowledge Injection, referred to as \method KI; \ding{175}~Full \method model.

Figure~\ref{fig:emb_resnet} and \ref{fig:emb_vib} show that IB methods produce more compact embeddings by minimizing $I(X;Z)$, reducing the encoding space and slightly lowering performance. Figure \ref{fig:emb_RIB_wo3} and \ref{fig:emb_RIB} illustrate that the third robust training stage of \method further restricts $I(\mathcal{D};S,T)$ and improves noise robustness. In particular, Figure \ref{fig:emb_RIB} shows embeddings \textbf{similar to IB’s minimal sufficient property} (Figure \ref{fig:emb_vib}), while clearer class boundaries demonstrate \methods ability to \textbf{learn cleaner representations}.

\begin{figure}[ht]
    \centering
    \begin{subfigure}[b]{0.45\linewidth}
        \includegraphics[width=\linewidth]{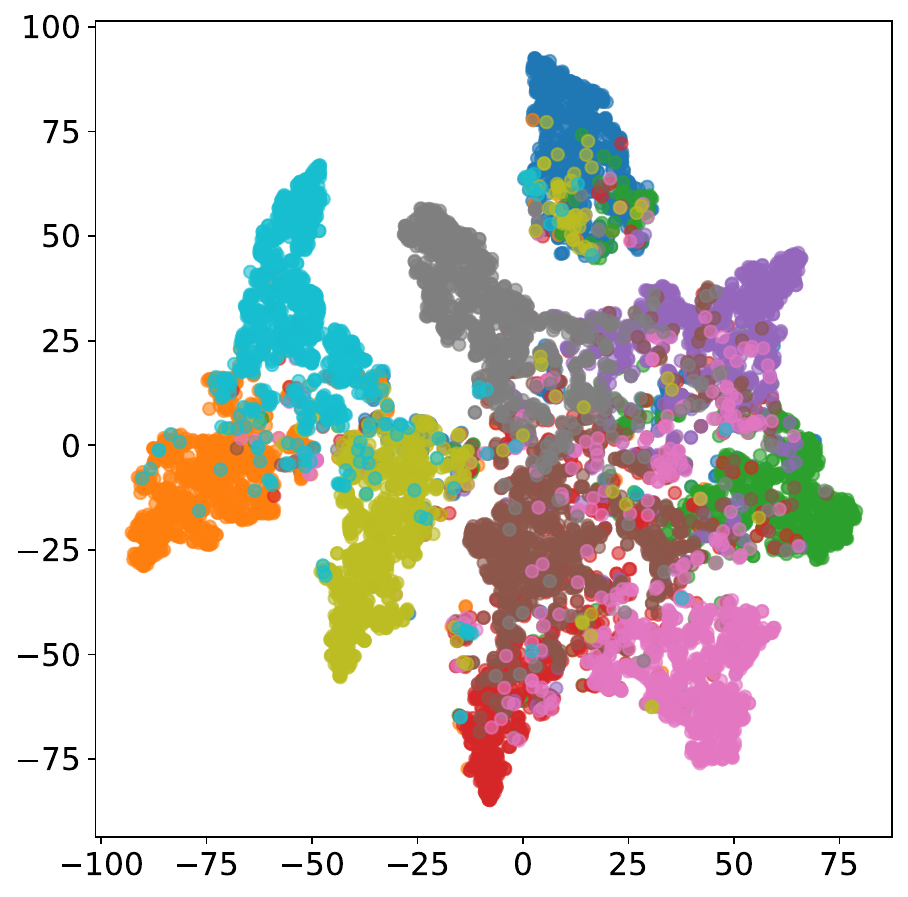}
        \caption{VIB w.o $I(X;Z)$}
        \label{fig:emb_resnet}
    \end{subfigure}
    \begin{subfigure}[b]{0.45\linewidth}
        \includegraphics[width=\linewidth]{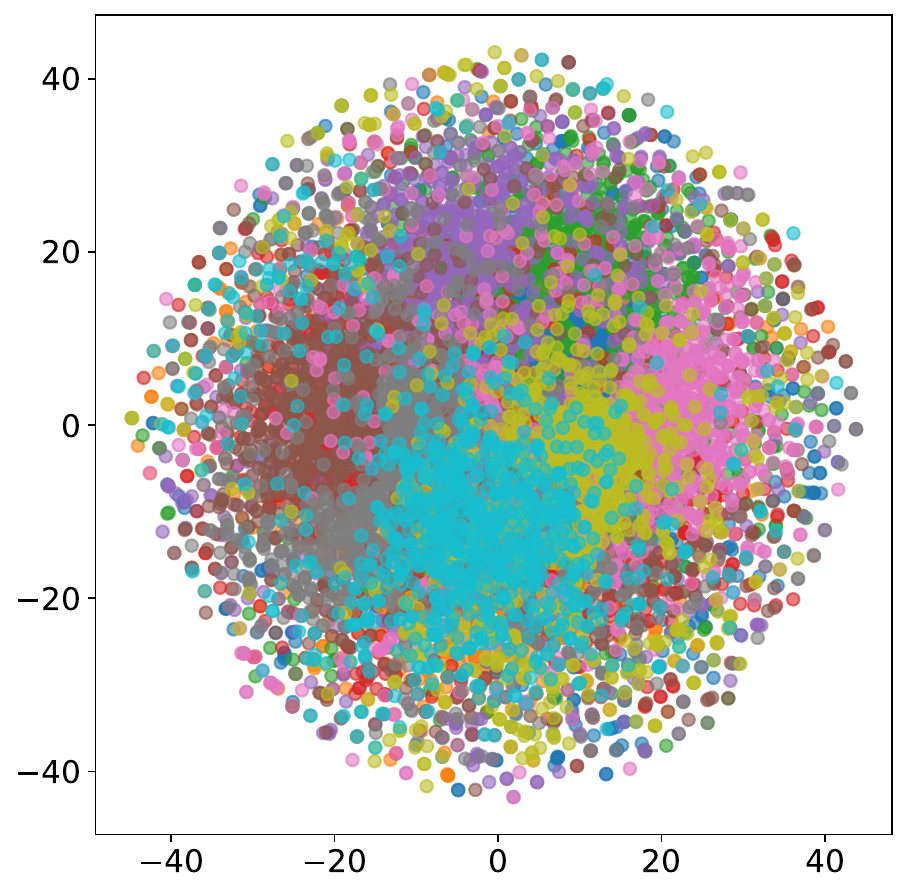}
        \caption{VIB}
        \label{fig:emb_vib}
    \end{subfigure}\\
    \begin{subfigure}[b]{0.45\linewidth}
        \includegraphics[width=\linewidth]{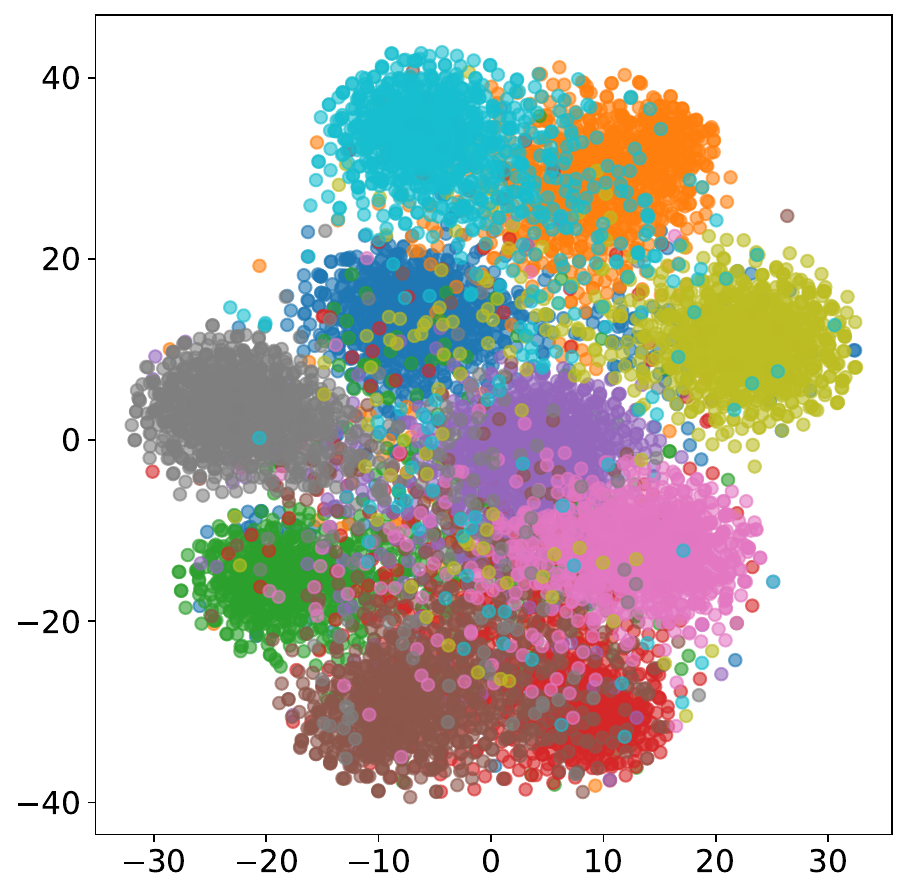}
        \caption{\method KI}
        \label{fig:emb_RIB_wo3}
    \end{subfigure}
    \begin{subfigure}[b]{0.45\linewidth}
        \includegraphics[width=\linewidth]{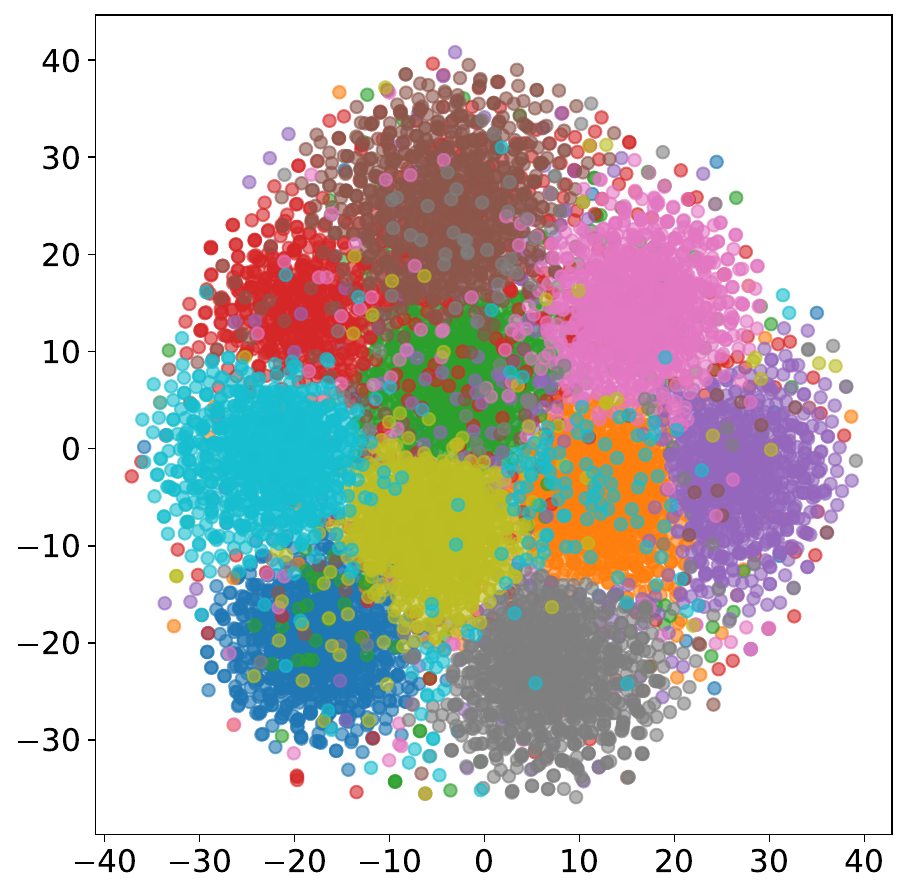}
        \caption{\method}
        \label{fig:emb_RIB}
    \end{subfigure}
\caption{The embedding distributions of different models}
\label{fig:emb}
\end{figure}

To investigate the learning process of the two encoders $S$ and $T$, Figure~\ref{fig:ST} shows their prediction accuracy on training and test sets with 40\% uniform noise on the Cora dataset, where vertical dashed lines divide the process into three periods. Encoder $T$ gradually fits all (noisy) data, while $S$, influenced by $T$, achieves about 65\% accuracy on the training set by fitting \textbf{mostly clean data}.

\begin{figure}[ht]
\centering
\includegraphics[height=4.5cm]{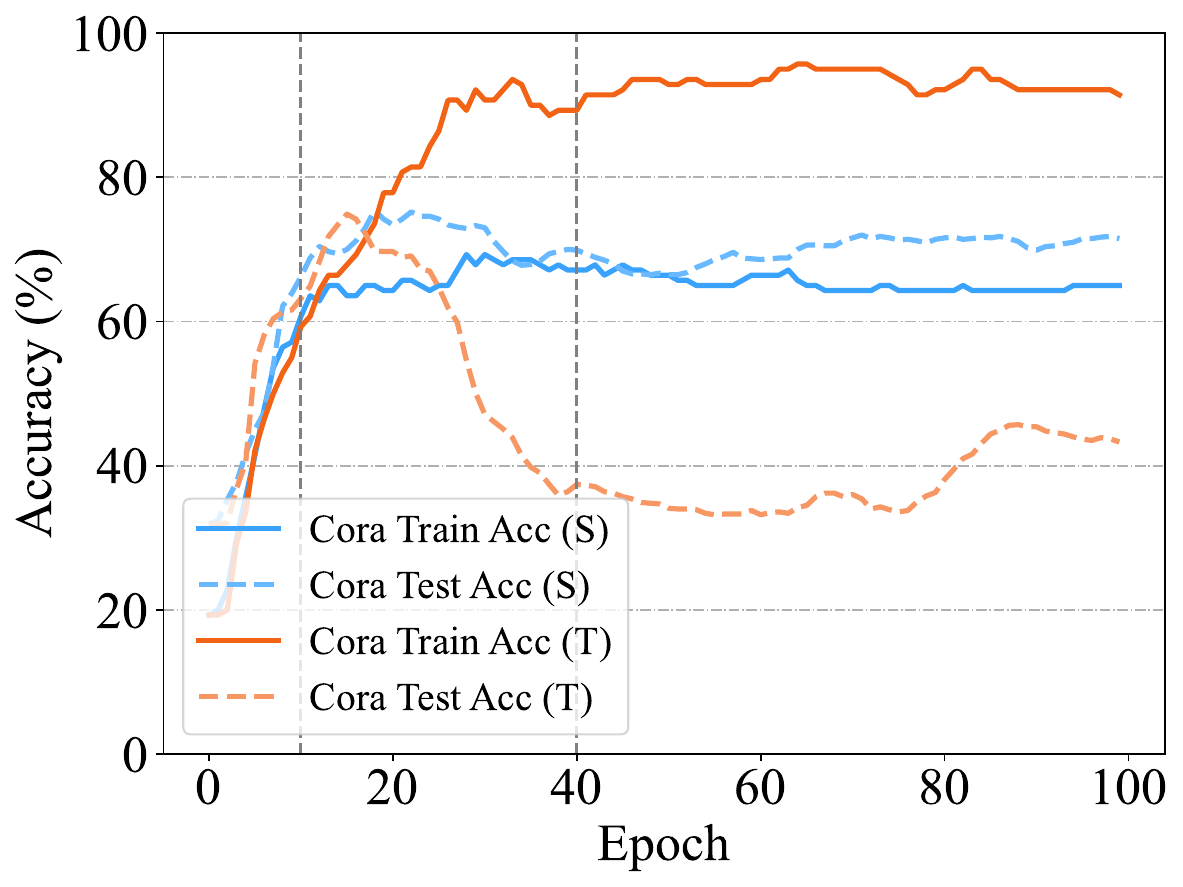}
\caption{The learning behavior of $(S, T)$}
\label{fig:ST}
\end{figure}

\subsection{Hyperparameter settings}
\label{appendix_hyper}
\subsubsection{Image Classification.}
For CIFAR-related datasets, we set the batch size to 256, and each image is reshaped to a size of (32, 32). For the Animal-10N dataset, the batch size is set to 128, and each image is reshaped to (68, 68). We use SGD as the optimizer with a learning rate of 0.005, momentum of 0.9, and a weight decay of 5e-4. A cosine learning rate scheduler is applied.

The dimension of the encoder's latent space is set to 128. Each experiment is repeated three times, and we report the mean and standard deviation of the results.

\subsubsection{Node Classification.}
For node classification tasks, we use the Adam optimizer with a learning rate of 0.001. To enable the computation of structural KL divergence, the model backbone is configured as a two-layer GAT. In our hyperparameter settings, the KL divergence weight for the features is set to 0.001, and the weight for the structural KL divergence is set to 0.01. The dimension of the encoder's latent space is set to 16 or 20. Each experiment is repeated three times, and we report the mean and standard deviation of the results.

In the experiments, the Algorithm~\ref{alg_sel} uses $\delta \in \{0.1, 0.2, \cdots, 0.9\}$. Besides, we set $\beta\in \{1e^{-1}, \cdots, 1e^{-5}\}$ and $\gamma = \in \{1e^{0}, \cdots, 1e^{-4}\}$ in our hyperparameter configuration. The predicted confidence scores are set based on the learning behavior of Warmup samples for each dataset, ensuring that the upper confidence bound is greater than 0.5, while the lower confidence bound is less than 0.5.

\subsection{Hardware and Software Configurations}
\label{appendix_conf}
We conduct the experiments with:
\begin{itemize}
    \item Operating System: Ubuntu 22.04.4 LTS.
    \item CPU: Intel(R) Xeon(R) Silver 4110 CPU @ 2.10GHz.
    \item GPU: NVIDIA Tesla V100 SMX2 with 32GB of Memory.
    \item Software: CUDA 12.8, Python 3.10.12, PyTorch\footnote{\url{https://github.com/pytorch/pytorch}} 2.2.0, PyTorch Geometric\footnote{\url{https://github.com/pyg-team/pytorch_geometric}} 2.6.1.
\end{itemize}

\end{document}